%% file: main.tex
\documentclass{article}
\usepackage{arxiv}
\input{math_commands.tex}

\usepackage{xcolor}
\definecolor{lightblue}{rgb}{0.6, 0.8, 0.9}
\definecolor{darkblue}{HTML}{0c58ac}
\definecolor{darkgreen}{rgb}{0, 0.55, 0.12}
\definecolor{darkred}{rgb}{0.6,0,0}
\definecolor{yellowbg}{RGB}{255,249,224}
\definecolor{greenbg}{RGB}{232,246,232}
\definecolor{bluebg}{RGB}{226,241,255}
\usepackage{natbib}
\usepackage[
  breaklinks=true,
  colorlinks,
  linkcolor=darkred,
  citecolor=darkgreen,
  bookmarks=false
]{hyperref}
\usepackage{url}

\usepackage{tikz}
\usepackage{graphicx}
\usepackage{wrapfig}
\usepackage{booktabs}
\usepackage{adjustbox}
\usepackage{multirow}
\usepackage{xspace}
\usepackage{array}
\usepackage{bm}
\usepackage{bbm}
\usepackage{xcolor}
\usepackage{amsthm, amsmath, amssymb}
\usepackage{cleveref}
\usepackage{algorithm}
\usepackage{algorithmic}

\usepackage{listings}
\usepackage{wrapfig}
\usepackage{caption}
\usepackage{subcaption}
\usepackage{url}
\usepackage{colortbl}
\usepackage{adjustbox}
\usepackage{makecell}
\usepackage{pifont}
\usepackage{tcolorbox}
\usepackage{setspace}
\usepackage{fancybox}
\usepackage{tocloft}
\usepackage{etoc}
\usepackage{enumitem}
\usepackage{pifont}
\usepackage{bm}
\definecolor{darksilver}{RGB}{160,160,160} 
\definecolor{mymauve}{rgb}{0.58,0,0.82}
\newlength{\mysize}

\input{macro-math.tex}

\title{Omnimodal Dataset Distillation via\\ High-order Proxy Alignment}

\author{
Yuxuan Gao$^{1,2*}$ \quad 
Xiaohao Liu$^{3}$\thanks{These authors contributed equally to this work and are listed alphabetically according to their last names.} \quad 
Xiaobo Xia$^{1}$\thanks{Corresponding author (xiaoboxia@ustc.edu.cn).} \quad
Tongliang Liu$^{4}$\\
$^1$University of Science and Technology of China \quad
$^2$Xidian University \\
$^3$National University of Singapore \quad 
$^4$The University of Sydney\\
}

\begin{document}

\maketitle

\begin{abstract}
\input{sec/0_abstract}
\end{abstract}
\addtocontents{toc}{\protect\setcounter{tocdepth}{-1}}

\input{sec/1_intro}

\input{sec/2_related_work}

\input{sec/3_method}
\input{sec/4_experiments}
\input{sec/5_conclusion}

\clearpage
\onecolumn

\renewcommand{\cftsecfont}{\normalsize} 
\renewcommand{\cftsubsecfont}{\normalsize} 
\renewcommand{\cftbeforesecskip}{12pt}      
\renewcommand{\cftbeforesubsecskip}{12pt}

\renewcommand{\contentsname}{Appendix}
\addtocontents{toc}{\protect\setcounter{tocdepth}{2}}

  \appendix
  \tableofcontents

    \clearpage
\input{sec/appendix}

\clearpage
\bibliography{main}
\bibliographystyle{unsrt}

\end{document}

%% file: math_commands.tex
\usepackage{amsmath,amsfonts,bm}

\def\eqref#1{equation~\ref{#1}}

\def\1{\bm{1}}

\def\vu{{\bm{u}}}
\def\vv{{\bm{v}}}

\def\vx{{\bm{x}}}

\def\vz{{\bm{z}}}

\def\mG{{\bm{G}}}

\def\mS{{\bm{S}}}

\DeclareMathAlphabet{\mathsfit}{\encodingdefault}{\sfdefault}{m}{sl}
\SetMathAlphabet{\mathsfit}{bold}{\encodingdefault}{\sfdefault}{bx}{n}

\def\gD{{\mathcal{D}}}

\def\gL{{\mathcal{L}}}
\def\gM{{\mathcal{M}}}



%% file: macro-math.tex
\allowdisplaybreaks

\newtheorem{thm}{Theorem}
\newtheorem{lem}[]{Lemma}

\theoremstyle{definition}

\newtheorem{assumption}[thm]{Assumption}

%% file: sec/0_abstract.tex
Dataset distillation compresses large-scale datasets into compact synthetic sets while preserving training performance, but existing methods are largely restricted to single-modal or bimodal settings. Extending dataset distillation to scenarios involving more than two modalities, \textit{i.e.}, Omnimodal Dataset Distillation, remains underexplored and challenging due to increased heterogeneity and complex cross-modal interactions.
In this work, we identify the key determinant that bounds the endpoint discrepancy in the omnimodal setting, which is exacerbated with an increasing number of modalities. To this end, we propose HoPA, a unified method that captures high-order cross-modal alignments via a compact proxy, which is compatible with trajectory matching as well. By abstracting omnimodal alignment with a shared similarity structure, our method avoids the combinatorial complexity of pairwise modality modeling and enables scalable joint distillation across heterogeneous modalities. Theoretical analysis from the spectral perspective reveals the rationality of our proposed method against bimodal dataset distillation techniques.
Extensive experiments on various benchmarks demonstrate that the proposed method achieves superior compression–performance trade-offs compared to existing competitors. The source code will be publicly released.

%% file: sec/1_intro.tex
\section{Introduction}
\label{sec:intro}
While the abundance of large-scale data has been a cornerstone of recent AI breakthroughs~\cite{kaplan2020scaling,team2023gemini,team2025kimi,achiam2023gpt,bahri2024explaining}, it is increasingly recognized that modern datasets often harbor a high degree of informational redundancy~\cite{sorscher2022beyond,lin2024not,lee2024coreset,xia2024rethinking,zhang2025survey}. 
Dataset Distillation, which has emerged as a particularly promising direction, condenses a large-scale dataset into a compact synthetic subset that preserves the critical information required for effective model training~\cite{wang2018dataset,zhao2020dataset,yu2025boost}. Unlike traditional pruning techniques~\cite{huang2024optimal,he2024large,zhang2023selectivity}, dataset distillation typically leverages optimization to generate synthetic data, achieving extremely high data compression while maintaining competitive performance, particularly in the context of large models and massive datasets~\cite{geng2023survey,liu2025evolution,li2026fixed}.

Most existing works of dataset distillation are limited to single-modal data, including vision~\cite{ding2025condensing,choi2026prism,dong2025high,shao2024elucidating}, text~\cite{nguyen2025synthetic,lu2025unidetox,shen2025condenselm}, and etc. 
More recently, a small number of studies have begun to extend vanilla dataset distillation to multimodal settings~\cite{xu2024low,zhao2025efficient,zhang2025beyond}, typically involving \textit{two} modalities, \textit{i.e.}, vision and text.
These works adopt trajectory matching to optimize a compact synthetic set by minimizing the discrepancy between a student model's parameter (endpoint) update path and a pre-recorded teacher trajectory derived from real data. 
Meanwhile, cross-modal similarity is highlighted as a vital abstraction within this paradigm, transitioning dataset distillation from simple data compression to the preservation of inter-modal alignments.
Despite the promising results, these studies are still confined to bimodal settings. 
Their specification on pairwise alignment limits the extension to complex scenarios involving more than two modalities in a unified and holistic manner. To this end, we formalize this frontier as Omnimodal Dataset Distillation, which necessitates the preservation across an arbitrary number of modalities. 

Distilling datasets that involve more than two modalities is inherently challenging. As the number of modalities increases, the heterogeneity of feature spaces, statistical properties, and semantic granularity across modalities becomes more pronounced~\cite{cicchetti2025gramian,liu2025principled}, making it difficult to construct a compact synthetic set that simultaneously preserves modality-specific information and cross-modal dependencies. Moreover, high-order interactions among multiple modalities introduce combinatorial complexity, which substantially complicates alignment, optimization, and consistency constraints during distillation. These factors collectively render dataset distillation beyond two modalities a non-trivial and largely underexplored problem, calling for the advent of Omnimodal Dataset Distillation techniques. 

In this paper, to address this research gap, we propose a unified omnimodal dataset distillation method via \textbf{H}igh-\textbf{o}rder \textbf{P}roxy \textbf{A}lignment (\textbf{HoPA}) that explicitly targets datasets with more than two modalities.
We first identify an overlooked factor in the alignment objective that determines the upper bound of the \textit{discrepancy} between the teacher model and optimized student endpoint under the omnimodal setting. An increased number of modalities introduces more mismatch terms into this bound, making omnimodal trajectory matching more challenging. This motivates a specialized design for omnimodal dataset distillation.
The inspired operationalization is to decouple omnimodal correspondence modeling from pairwise modality interactions by constructing a compact proxy that captures high-order cross-modal relationships among all modalities. 
This proxy is derived from multimodal similarity reconstruction by the spectral analysis and is naturally compatible with trajectory matching. Moreover, this design avoids the combinatorial explosion induced by enumerating all modality pairs, enabling efficient and scalable distillation across heterogeneous modalities while preserving both modality-specific information and global cross-modal consistency.
We further provide a theoretical analysis on spectral selectivity to validate the effectiveness of our proposed method compared to dataset distillation techniques with pairwise objectives. Empirically, performance comparison against the state-of-the-art (SOTA) baselines across diverse omnimodal datasets and data ratios demonstrates superior efficacy and efficiency. Generalization is validated across different architectures, while ablation studies and further analysis reveal the rationality of our proposed method.
Before delving into details, we summarize our contributions as follows:
\begin{itemize}
    \item Conceptually, we identify and study, for the first time, the problem of omnimodal dataset distillation for data involving more than two modalities, highlighting a critical limitation of existing single-modal and bimodal dataset distillation paradigms.
    \item Technically, building upon the analysis of endpoint discrepancy, we propose HoPA, a unified omnimodal dataset distillation method that captures high-order cross-modal relationships via a compact proxy, enabling scalable and effective distillation without relying on exhaustive pairwise modality interactions. Theoretical comparison against the bimodal dataset distillation proves a tighter trajectory bound introduced by our design.
    \item Empirically, we conduct extensive experiments on omnimodal benchmarks to demonstrate the effectiveness and scalability of the proposed method. It consistently achieves better compression-performance trade-offs than existing baselines. Detailed ablations and case studies are also provided.
\end{itemize}

%% file: sec/2_related_work.tex
\section{Related Work}
\label{sec:related_work}
In this section, we review recent literature related to this
work, including dataset distillation and omnimodal alignment. 
\subsection{Dataset Distillation}
Dataset distillation seeks to compress a large-scale dataset into a compact synthetic set that can effectively substitute the original data for model training, while preserving comparable performance~\citep{wang2018dataset,yu2023dataset,lei2023comprehensive,fang2026knowledge,zhong2025hierarchical,cui2025optical,su2026diffusion}. Early work has mainly focused on single-modal data distillation, with a particular emphasis on image data. These methods can be broadly categorized into several representative paradigms. 
Gradient matching methods optimize synthetic samples by aligning the training gradients induced by real and synthetic data~\citep{zhao2020dataset,lee2022dataset,kim2022dataset}.
Trajectory matching methods instead preserve the optimization dynamics of training on the original dataset by matching parameter trajectories or intermediate model states~\citep{cazenavette2022dataset,guo2024towards,zhong2025towards}.
Distribution matching methods align feature-level or statistical distributions between real and distilled data~\citep{liu2025dataset,cui2025optical,zhao2023dataset}.
Meta-learning–based methods formulate the distilled set as learnable data optimized for fast adaptation or strong generalization across training episodes~\citep{nguyen2020dataset,such2020generative,deng2022remember,zhou2022dataset,chen2024provable}.
Generative methods have introduced pretrained or trainable generators, such as generative adversarial networks~(GANs) and diffusion models, to parameterize the distilled data and improve cross-architecture generalization~\citep{cazenavette2023generalizing,gu2024efficient,su2024d}.

Recent efforts have attempted to extend dataset distillation to multimodal scenarios, \textit{e.g.}, vision-language~\citep{wu2023vision,xu2024low,lee2025covmatch} and vision-audio settings~\citep{kushwaha2024audio,li2025decoupled}. However, existing multimodal distillation methods typically handle only two modalities at a time. Extending them to more than two modalities often requires repeated or pairwise distillation, which is inefficient and poorly scalable. This limitation motivates our work, which performs joint distillation of multiple modalities in a single unified process.

\subsection{Omnimodal Alignment}
Omnimodal alignment has emerged as a fundamental problem in multimodal learning, to learn coherent and consistent representations across a diverse set of modalities. 
Representative works in this direction include ImageBind~\citep{girdhar2023imagebind}, LanguageBind~\citep{zhu2024languagebind}, and more recent studies on scalable and principled cross-modal representation learning~\citep{cicchetti2025gramian,wang2024omnibind,liu2025principled,luo2025next,zhang2023universal,zong2024self,jiang2026representation,jin2024hierarchical,li2026omni}.
These methods typically seek to map heterogeneous inputs into a shared embedding space in which semantically related samples from different modalities are well aligned. 
Such an alignment provides a strong foundation for cross-modal transfer, retrieval, and joint understanding.

Different from prior omnimodal alignment methods, which are mainly concerned with representation learning from full training datasets, our work leverages omnimodal alignment as a surrogate objective for dataset distillation. 
The central goal is not merely to learn a shared representation space, but to preserve the cross-modal relational structure of the original omnimodal dataset using a highly compact distilled set.

%% file: sec/3_method.tex
\section{Preliminary}
\label{sec:preliminary}
\textbf{Omnimodal similarity.}
Given an input instance with multiple modalities, we denote by $E_m(\cdot)$ the encoder for modality $m$ and by
$\vz_m = E_m(\vx_m) \in \mathbb{R}^d$ the corresponding representation.
Following common practice, all representations are $\ell_2$-normalized such that $\|\vz_m\|_2 = 1$.
For a bimodal pair $(m_1, m_2)$, the similarity is measured by the cosine similarity, \textit{i.e.},
$s(\vx_{m_1}, \vx_{m_2}) = \langle \vz_{m_1}, \vz_{m_2} \rangle$, which characterizes the alignment strength between the two modalities for the same instance.
The bimodal similarity can be extended to the omnimodal setting by jointly considering all modalities associated with a single instance.
Let $\gM=\{m_1,\ldots,m_k\}$ denote the modality set and
$\vz=[\vz_{m_1},\ldots,\vz_{m_k}]^\top \in \mathbb{R}^{k\times d}$.
The pairwise similarities among modalities can be compactly represented by a Gram matrix, formally represented as: 
\begin{equation}
\mG = \vz\vz^\top
= [\langle \vz_i, \vz_j \rangle]_{i,j\in\gM} .
\end{equation}
The Gram matrix encodes the complete correlation structure of all modalities within one instance and serves as a unified descriptor of omnimodal alignment.

\textbf{Negatives for omnimodal alignment.}
While $\mG$ characterizes positive correspondences across modalities of the same instance, effective alignment learning also requires negative samples.
In the bimodal case, negatives can be obtained by cross-instance similarities
$\langle \vz^{(i)}_{m_1}, \vz^{(j)}_{m_2} \rangle$ for $i\neq j$.
In the omnimodal scenario, directly constructing negatives for all modality pairs leads to a similarity tensor of size $\mathcal{O}(k^2 N^2)$, which is computationally prohibitive.
This observation motivates a compact proxy that preserves discriminative alignment information while avoiding explicit enumeration of all negative modality pairs.

\textbf{Trajectory matching.}
A common paradigm in dataset distillation is to learn a compact synthetic set $\mathcal{D}_{\text{syn}}$ whose induced training dynamics match those induced by the real data~\cite{liu2024dataset,cazenavette2022dataset,lee2024selmatch}.
Concretely, we first train an expert model on real data under an inner objective $\gL_{\text{inner}}$, and store its parameter trajectory in an expert buffer.
During distillation, a teacher segment is sampled from the buffer, where $\boldsymbol{\theta}_0$ denotes the start point, and $\boldsymbol{\theta}_T$ denotes the target parameter after $T$ training steps on real data.
The student model is then initialized from the same start point $\boldsymbol{\theta}_0$ and optimized on the synthetic set ${\mathcal{D}}_{\text{syn}}$ for $t$ steps, yielding a student endpoint $\boldsymbol{\theta}_e^t$.
The synthetic set is learned by minimizing the discrepancy between the student endpoint and the teacher target:
\begin{equation}
\label{eq:traj_matching_prelim}
\gL_{\text{traj}}
=\frac{\|\boldsymbol{\theta}_e^{t}-\boldsymbol{\theta}_T\|_2^2}
{\|\boldsymbol{\theta}_0-\boldsymbol{\theta}_T\|_2^2},
\qquad
\mathcal{D}_{\text{syn}}=\arg\min_{\mathcal{D}_{\text{syn}}}\gL_{\text{traj}}.
\end{equation}
This objective encourages the synthetic data to reproduce the local training trajectory of the real data on the student model. It provides a generic optimization method that can be instantiated with different distillation objectives.

\textbf{Bimodal dataset distillation.}
In the bimodal setting, the goal is not only to compress modality-specific content but also to preserve cross-modal correspondence. 
Existing trajectory-based methods typically optimize the two modality branches jointly using the same synthetic set, while measuring the trajectory discrepancy separately for each encoder branch.
Concretely, building on Eq.~(\ref{eq:traj_matching_prelim}), the bimodal objective can be written as 
\begin{equation}
\label{eq:bi_traj_prelim}
\gL_{\text{bi-traj}}
=
\sum_{m\in\{m_a,m_b\}}
\frac{\|\boldsymbol{\theta}_{e,m}^{t}-\boldsymbol{\theta}_{T,m}\|_2^2}
{\|\boldsymbol{\theta}_{0,m}-\boldsymbol{\theta}_{T,m}\|_2^2},
\end{equation}
where $m$ indexes modality branches, and each term measures the normalized trajectory mismatch between student and expert endpoints for that modality.

However, the pairwise formulation becomes inadequate in the omnimodal setting. 
A direct extension to $k>2$ modalities would require handling a quadratic number of modality pairs, leading to substantially increased optimization complexity.
More importantly, pairwise objectives decompose one omnimodal instance into a collection of independent bimodal relations, and therefore do not explicitly capture the holistic correlation structure jointly shared by all modalities.
In this work, we retain the branch-wise trajectory matching, but replace the pairwise alignment view with a compact omnimodal objective that models all modalities of an instance as a unified whole.

\section{Methodology}\label{sec:method}
In this section, we present our method step by step. We first analyze the theoretical underpinnings of trajectory mismatch, where the discrepancy between teacher and student parameters is upper-bounded by the accumulation of modality-wise gradient mismatches (Sec.~\ref{sec:traj_friendly}). Inspired by this, we introduce HoPA, a method that moves beyond suboptimal pairwise constraints by treating all modalities as a unified whole (Sec.~\ref{sec:pipeline}). We derive a compact, rank-1 semantic proxy from the leading singular components of the Gram matrix, which serves as a tractable representation for omnimodal alignment. This design provably tightens the trajectory-matching bound compared to full-spectrum pairwise objectives (Sec.~\ref{sec:spectral_criterion}). Finally, we provide a comprehensive gradient analysis to elucidate how the proposed objective functions through a synergistic balance of intra-instance centripetal alignment and inter-instance distinctive separation (Sec.~\ref{sec:gradient_analysis}).

\subsection{Inner Objective and Trajectory Mismatch}
\label{sec:traj_friendly}

Trajectory matching provides a generic optimization method for dataset distillation.
Prior work has mainly focused on improving trajectory construction and matching efficiency~\citep{du2023minimizing,liu2024dataset}.
In contrast, we emphasize that the inner training objective $\gL_{\text{inner}}$ serves as a key factor. It determines the gradient field induced by the teacher objective and therefore directly affects the upper bound of the endpoint discrepancy between teacher and student parameters.
To make this explicit, let $\gM$ denote the modality set with $|\gM|=k$, and let
$\boldsymbol{\theta}_r^\text{T}=\{\boldsymbol{\theta}_{r,m}^\text{T}\}_{m\in\gM}$,
$\boldsymbol{\theta}_r^\text{S}=\{\boldsymbol{\theta}_{r,m}^\text{S}\}_{m\in\gM}$.
We consider the rollout, 
\[
\boldsymbol{\theta}_{r+1,m}^\text{T}
=\boldsymbol{\theta}_{r,m}^\text{T}-\eta\,\mathbf{g}_{T,m}(\boldsymbol{\theta}_r^\text{T}),\quad
\boldsymbol{\theta}_{r+1,m}^\text{S}
=\boldsymbol{\theta}_{r,m}^\text{S}-\eta\,\mathbf{g}_{S,m}(\boldsymbol{\theta}_r^\text{S}),\quad r=0,\ldots,n-1,
\]
with shared initialization $\boldsymbol{\theta}_0^\text{T}=\boldsymbol{\theta}_0^\text{S}$.
Here, both $\mathbf{g}_{T,m}$ and $\mathbf{g}_{S,m}$ are gradients induced by the inner objective $\gL_{\text{inner}}$ on real and synthetic data, respectively.

\begin{lem}[Local gradient mismatch upper-bounds endpoint discrepancy]
\label{lem:local_to_segment_main}
Let $\mathbf g_T(\boldsymbol{\theta}) := \{\mathbf g_{T,m}(\boldsymbol{\theta})\}_{m\in\gM}$, $\mathbf g_S(\boldsymbol{\theta}) := \{\mathbf g_{S,m}(\boldsymbol{\theta})\}_{m\in\gM}$, denote the stacked gradient fields over all modality branches, equipped with the block Euclidean norm.
Assume $\mathbf{g}_T$ is locally $L$-Lipschitz along the student rollout, i.e.,
$\|\mathbf{g}_T(\boldsymbol{\theta}_r^\text{S})-\mathbf{g}_T(\boldsymbol{\theta}_r^\text{T})\|
\le L\|\boldsymbol{\theta}_r^\text{S}-\boldsymbol{\theta}_r^T\|,\ \forall r$.
Then the endpoint discrepancy is upper-bounded as
\begin{equation}
\small
\|\boldsymbol{\theta}_n^\text{S}-\boldsymbol{\theta}_n^\text{T}\|
\le
\eta\sum_{r=0}^{n-1}(1+\eta L)^{n-1-r}
\|\mathbf g_S(\boldsymbol{\theta}_r^\text{S})-\mathbf g_T(\boldsymbol{\theta}_r^\text{S})\| =
\eta\sum_{r=0}^{n-1}(1+\eta L)^{n-1-r}\Big(\sum_{m\in\mathcal M}
\|\mathbf g_{S,m}(\boldsymbol{\theta}_r^\text{S})-\mathbf g_{T,m}(\boldsymbol{\theta}_r^\text{S})\|_2^2
\Big)^{1/2}.
\label{eq:graident}
\end{equation}
\end{lem}

\begin{tcolorbox}
    [colframe=gray!80, colback=gray!10, boxrule=0.2pt, arc=5pt, boxsep=2.5pt, left=2pt, right=2pt, top=2pt, bottom=2pt] \textbf{Remark.} 
Lemma~\ref{lem:local_to_segment_main} shows that the endpoint discrepancy is upper bounded by the accumulated modality-wise gradient mismatch induced by $\gL_{\text{inner}}$, see Eq.~(\ref{eq:graident}).
Compared with the bimodal case (\textit{i.e.}, $k=2$), increasing $k$ introduces more modality-wise mismatch terms into the bound, making multimodal trajectory matching more challenging in the multimodal setting under a fixed synthetic budget.
Therefore, designing a suitable inner objective is essential for tightening the bound (See Appendix~\ref{app:proof_lemma1} for the full derivation).
\end{tcolorbox}

This analysis suggests two implications for the objective design: 
(1) A preferable inner objective should induce teacher and student gradients that minimize the upper bound of the endpoint discrepancy; (2) This challenge is exacerbated in the multimodal setting, where modality-wise mismatch terms accumulate across branches. 
These observations motivate our design of the inner objective in the following sections, which facilitates gradient matching within the multimodal setting.

\begin{figure*}
    \centering
    \includegraphics[width=\linewidth]{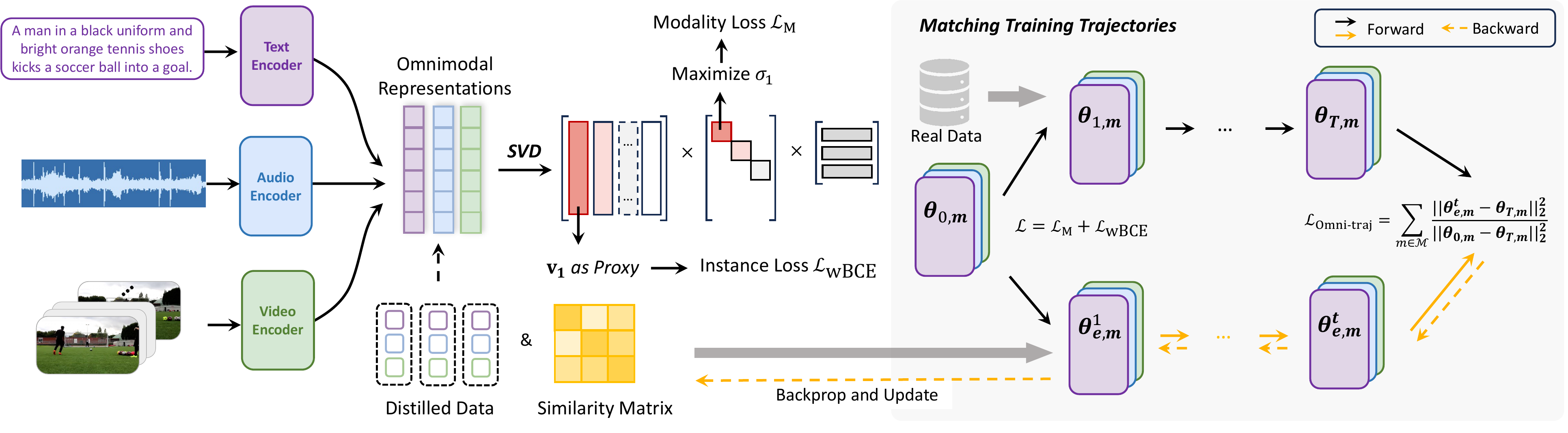}
    \caption{\textbf{The overall framework of omnimodal dataset distillation beyond bimodal data.} 
    Different modal contents are encoded into multimodal representations (\textit{left}).  
    We decompose these representations via SVD, rather than combining pairwise modeling, to handle the increased heterogeneity among modalities. 
    We maximize the leading singular value $\sigma_1$ and utilize the principal right singular vector $v_1$ as a compact low-rank proxy to guide the alignment (\textit{middle}). 
    We distill the dataset into the synthetic set by matching the training trajectories. A shared similarity matrix $\mS_e$ is jointly learned to preserve the multimodal alignment within the distilled data (\textit{right}).}
    \label{fig:framework}
\end{figure*}

\subsection{HoPA: High-order Proxy Alignment for Omnimodal Dataset Distillation}
\label{sec:pipeline}

Motivated by the above analysis, we propose the HoPA that moves beyond pairwise constraints by treating the set of all modalities as a unified whole.
To operationalize this, we introduce a compact proxy derived from the optimal approximation of the Gram matrix. This proxy is capable of preserving the structural alignment of multimodal similarities and is compatible with trajectory matching. This allows for minimal adjustments when extending bimodal trajectory matching to a multimodal context without incurring the computational cost of full matching.
The full procedure is shown in Appendix~\ref{sec:algorithm_flow}.
Furthermore, we provide a comprehensive analysis to guarantee technical rigor and demonstrate the tighter bounds (see Sec.~\ref{sec:spectral_criterion}) achieved by our method through theoretical grounding.

\textbf{Multimodal similarity reconstruction.}
The Gram matrix $\mG$ encodes the pairwise correlations between all modalities, admitting a spectral decomposition
$\mG = \vz\vz^\top
= \sum_{i=1}^{k} \sigma_i^2 \vu_i \vu_i^\top$,
where $\sigma_1 \ge \sigma_2 \ge \cdots \ge 0$ are the singular values of $\vz$, and $\vu_i \in \mathbb{R}^k$ are the corresponding eigenvectors, \textit{i.e.}, the left singular vectors of $\vz$. 

\begin{lem}[Rank-1 optimal approximation~\citep{liu2025principled}]
Let $\mG \in \mathbb{R}^{k \times k}$ be the Gram matrix of the omnimodal representations. The rank-1 matrix $\tilde{\mG} = \sigma_1^2 \vu_1 \vu_1^\top$ is the unique optimal approximation of $\mG$ in terms of the Frobenius norm, satisfying:
\begin{equation}
\tilde{\mG} = \underset{\mathbf{A} \in \mathbb{R}^{k \times k}, \text{rank}(\mathbf{A})=1}{\arg\min} \|\mG - \mathbf{A}\|_F.
\end{equation}
\end{lem}
\begin{tcolorbox}
    [colframe=gray!80, colback=gray!10, boxrule=0.2pt, arc=5pt, boxsep=2.5pt, left=2pt, right=2pt, top=2pt, bottom=2pt] \textbf{Remark.} 
Accordingly, the optimal rank-1 approximation under the Frobenius norm is given by
$\sigma_1^2 \vu_1 \vu_1^\top$.
When representations from different modalities are well aligned, the leading eigenvalue dominates, and the reconstruction error becomes negligible,  \textit{i.e.}, $\|\mG - \sigma_1^2\vu_1^\top\vu_1\|_F \to 0$.
\end{tcolorbox}

Hinged with the omnimodal distillation setting, we propose to extract the omnimodal alignment holistically, rather than the direct distillation from multiple pair-wise similarities. Therefore, we derive the following lemma. 
\begin{lem}[Modality-feature duality]
\label{lem:duality}
Let $\vv_1 \in \mathbb{R}^d$ be the principal right singular vector of $\vz$. The relationship between the modality alignment vector $\vu_1$ and the unified semantic feature $\vv_1$ is given by: $\sigma_1 \vv_1 = \vz^\top \vu_1 = \sum_{m \in \gM} [\vu_1]_m \vz_m$.
\end{lem}
\begin{tcolorbox}[colframe=gray!80, colback=gray!10, boxrule=0.2pt, arc=5pt, boxsep=2.5pt, left=2pt, right=2pt, top=2pt, bottom=2pt] \textbf{Remark.} 
This duality implies that $\vu_1$ acts as an optimal weight vector that aggregates diverse modality representations. $\vv_1$ serves as a weighted aggregation of all modality representations $\{\vz_m\}$. Unlike simple averaging, this spectral aggregation automatically assigns higher importance to modalities that are more consistent with the dominant semantic direction. To this end, we propose using the proxy for the dataset distillation, adopting one global scale for the multimodal similarity rather than a full Gram matrix, yielding omnimodal trajectory matching as well.
\end{tcolorbox}

\textbf{\ding{172} Eigenvector as proxy.}
Rather than operating on the full Gram matrix, we adopt the leading eigenvector as a low-dimensional proxy for omnimodal similarity. And we derive the following proposition.
Let $\vv_1$ denote the right singular vector associated with $\sigma_1$.
For two instances indexed by $i$ and $j$, we approximate their omnimodal similarity by
\begin{equation}
\tilde{s}(\vx_i,\vx_j) = \vv_1^{(i)\top} \vv_1^{(j)}.
\end{equation}
This approximation also compresses the original $\mathcal{O}(k^2)$ modality interactions into a single scalar similarity while preserving the dominant alignment structure.
As a result, large multimodal similarity tensors are reduced to a tractable instance-level similarity matrix.

\textbf{\ding{173} Omnimodal trajectory matching.}
Using the eigenvector-based proxy, we match the training trajectory induced by synthetic data to that of real data.
Specifically, we construct a proxy similarity matrix $\vv_1^{(i)\top} \vv_1^{(j)}$ for different instances $i,j$ and enforce consistency between the optimization dynamics driven by synthetic and real samples.
This trajectory matching principle ensures that the distilled data induces gradient updates aligned with those of the original multimodal dataset. We formulate multimodal distillation as a bilevel optimization problem.
Let $\boldsymbol{\theta}_T$ denote the model parameters trained on real data and $\boldsymbol{\theta}_e$ those trained on synthetic data $(\gD_{\text{syn}},\mS_e)$, where $\mS_e$ is the synthetic similarity matrix that follows the same structure as the eigenvector-based proxy.
We retain the branch-wise trajectory matching used in bimodal dataset distillation, and extend it to the omnimodal setting over all modality branches.
\begin{equation}
\label{eq:omni_traj}
\gL_{\text{omni-traj}}
=
\sum_{m\in\gM}
\frac{\|\boldsymbol{\theta}_{e,m}^{t}-\boldsymbol{\theta}_{T,m}\|_2^2}
{\|\boldsymbol{\theta}_{0,m}-\boldsymbol{\theta}_{T,m}\|_2^2},
\quad
(\gD_{\text{syn}},\mS_e)
=
\arg\min_{\gD_{\text{syn}},\mS_e}
\gL_{\text{omni-traj}},
\end{equation}
Model training is driven by a combined objective consisting of modality-level contrastive learning and instance-level alignment regularization:
\begin{equation}
\label{eq:inner_target}
\gL = \gL_\mathrm{M}+ \gL_{\mathrm{wBCE}}(\mS_e),\quad \gL_\mathrm{M}
= -\frac{1}{N}\sum_{i=1}^N
\log \frac{\exp(\sigma_1^{(i)}/\tau)}{\sum_\ell \exp(\sigma_\ell^{(i)}/\tau)}.
\end{equation}
The weighted binary cross-entropy loss $\gL_{\mathrm{wBCE}}$ enforces consistency between $\mS_e$ and the proxy similarities derived from $\vv_1$,
thereby aligning instance-level gradients across real and synthetic data. Specifically, we follow~\cite{xu2024low} to implement $\gL_{\mathrm{wBCE}}$, but instead use the proposed proxies: 
{\small
\begin{align}
    \gL_{\text{wBCE}}(\mS_e) =& \frac{1}{|\{s_{ij}> \beta\}|} \sum_{i,j: s_{ij}> \beta} \ell\left(s_{ij}, \mathrm{sigmoid}(\vv_1^{(i)\top}\vv_1^{(j)}/\tau')\right) +\frac{1}{|\{s_{ij}\leq \beta\}|} \sum_{i,j: s_{ij}\leq \beta} \ell\left(s_{ij}, \mathrm{sigmoid}(\vv_1^{(i)\top}\vv_1^{(j)}/\tau')\right),
\end{align}
}
where $\ell(y,\hat{y}) = -y\log\hat{y} - (1-y)\log(1-\hat{y})$, $\beta$ is the positive/negative threshold and set to 0.5, $s_{ij}$ is the learnable similarity of the $i$-th instance and $j$-th instance. 
Through this gradient-based matching, the distilled data preserves the essential omnimodal alignment behavior of the original dataset.

\subsection{Spectral Selectivity in Trajectory Matching}
\label{sec:spectral_criterion}
Lemma~\ref{lem:local_to_segment_main} shows that trajectory matching quality is affected by the accumulated gradient mismatch, which is in turn determined by the choice of inner objective.
This motivates a more structural question: \textit{Beyond empirical performance, what property of an inner objective makes teacher--student trajectories intrinsically easier to match in the omnimodal regime?}
To answer this, we analyze objectives through the eigenspectrum of the Gram matrix $\mathbf{G}=\mathbf{z}\mathbf{z}^\top=\sum_{j=1}^{k}\lambda_j\mathbf{u}_j\mathbf{u}_j^\top$, where $\lambda_j=\sigma_j^2$. 
This analysis also supports the rationality and effectiveness behind our method design.
The key claim is that spectral selectivity is the critical factor. Pairwise losses are full-spectrum and typically activate all singular modes, while our proposed method optimizes only the dominant mode (through $\sigma_1$ and its proxy).
As a consequence, our method removes mismatch contributions from non-principal modes and yields a tighter certified upper bound on endpoint discrepancy than pairwise objectives whenever high-order modes are non-negligible.

\textbf{Spectral mismatch model.}
Lemma~\ref{lem:local_to_segment_main} reduces the problem of bounding the endpoint discrepancy to controlling the per-step gradient mismatch $\Delta_r = \|\mathbf{g}_S(\boldsymbol{\theta}_r^\text{S}) - \mathbf{g}_T(\boldsymbol{\theta}_r^\text{S})\|_2$ at each rollout step $r$.
We now characterize how $\Delta_r$ depends on the eigenspectrum of the inner objective.
For a spectral objective $\gL=f(\sigma_1,\ldots,\sigma_k)$, standard matrix calculus gives $\nabla_{\mathbf{z}}\sigma_j=\mathbf{u}_j\mathbf{v}_j^\top$ at points of distinct singular values~\citep{ionescu2015matrix}, and by the chain rule 
$\nabla_{\mathbf{z}}f
=\sum_{j=1}^{k}
\frac{\partial f}{\partial\sigma_j}
\mathbf{u}_j\mathbf{v}_j^\top$, 
where $(\partial f/\partial\sigma_j)|_r:=\alpha_{j,r}$ is the sensitivity of the objective to the $j$-th singular value at step $r$.
Let $\varepsilon_{j,r}:=\|\mathbf{u}_j^\text{S}\mathbf{v}_j^{\text{S}\top}-\mathbf{u}_j^\text{T}\mathbf{v}_j^{\text{T}\top}\|_F$ denote the mode-$j$ approximation error at step $r$, the per-step mismatch satisfies,
\begin{equation}
\label{eq:spectral_mismatch_main}
\Delta_r
\;\le\;
C\sum_{j=1}^{k}|\alpha_{j,r}|\,\varepsilon_{j,r},
\end{equation}
where $C>0$ absorbs the encoder Jacobian norm $\|\mathbf{J}\|$ (see Appendix~\ref{app:spectral_mismatch_derivation} for the full derivation).

\textbf{Pairwise losses are full-spectrum objectives.}
For a pairwise objective $\mathcal{L}_{\mathrm{pair}}$, the sensitivity to the $j$-th spectral mode of instance $i$ is measured by the projection of the gradient onto the $j$-th singular direction:
\begin{equation}
\beta_j^{(i)}
:= \left\langle
     \nabla_{\mathbf{z}^{(i)}}\mathcal{L}_{\mathrm{pair}},\;
     \mathbf{u}_j^{(i)}(\mathbf{v}_j^{(i)})^\top\right\rangle_F
 = \mathbf{u}_j^{(i)\top} \bigl(\nabla_{\mathbf{z}^{(i)}}\mathcal{L}_{\mathrm{pair}}\bigr) \mathbf{v}_j^{(i)},
\label{eq:beta_proj}
\end{equation}
where the equality follows from the definition of the Frobenius inner product, and $\mathbf{u}_j^{(i)}(\mathbf{v}_j^{(i)})^\top$ is the rank-1 matrix associated with the $j$-th singular mode of $\mathbf{z}^{(i)}$.
The objective is \textit{full-spectrum} if $\beta_j^{(i)}\neq 0$ for all $j$ in general.
Standard pairwise objectives, including pairwise InfoNCE and pairwise weighted BCE, fall into this class (see Appendix~\ref{app:spectralize_pairwise_losses}), and therefore accumulate gradient mismatch from all active spectral modes.
By contrast, our proposed objective depends on $\mathbf{z}^{(i)}$ only through $\sigma_1^{(i)}$, making $\beta_j^{(i)}=0$ for all $j\geq 2$ by construction.

Based on this, we compare two classes of inner objectives: a \textit{single-mode} objective $A$ that depends on $\mathbf{z}$ only through $\sigma_1$ (ours), and a \textit{full-spectrum} objective $B$ that depends on all $\{\sigma_j\}_{j=1}^k$.
Their induced spectral weights at rollout step $r$ are denoted $\{\alpha^A_{j,r}\}$ and $\{\alpha^B_{j,r}\}$ respectively.

\begin{thm}[Eigenvalue selectivity yields tighter trajectory bounds]
    \label{thm:eigen_selectivity}
    Under Lemma~\ref{lem:local_to_segment_main} and the spectral mismatch model Eq.~(\ref{eq:spectral_mismatch_main}), the endpoint discrepancy satisfies 
    \begin{equation}
    \label{eq:bound_general_main}
    \|\boldsymbol{\theta}_n^\text{S}-\boldsymbol{\theta}_n^\text{T}\|_2
    \;\le\;
    \eta C\sum_{r=0}^{n-1}(1+\eta L)^{n-1-r}
    \sum_{j=1}^{k}|\alpha_{j,r}|\,\varepsilon_{j,r}.
    \end{equation}
    Suppose $\alpha^A_{1,r}=\alpha^B_{1,r}$ for all $r$, with $\alpha^A_{j,r}=0$ for all $j\ge2$, while $\alpha^B_{j,r}\neq0$ for some $j\ge2$.
    Then their certified upper bounds satisfy $U_A \le U_B$, where
    \begin{align}
    U_A = \eta C\sum_{r=0}^{n-1}(1+\eta L)^{n-1-r}
    |\alpha^A_{1,r}|\,\varepsilon_{1,r}, \quad
    U_B = \eta C\sum_{r=0}^{n-1}(1+\eta L)^{n-1-r}
    \sum_{j=1}^{k}|\alpha^B_{j,r}|\,\varepsilon_{j,r}.
    \end{align}
    The inequality is strict whenever $\alpha^B_{j,r}\varepsilon_{j,r}\neq0$ for some $j\ge2$ at some step $r$.
\end{thm}

\begin{tcolorbox}
    [colframe=gray!80, colback=gray!10, boxrule=0.2pt, arc=5pt, boxsep=2.5pt, left=2pt, right=2pt, top=2pt, bottom=2pt] \textbf{Remark.} 
    Theorem~\ref{thm:eigen_selectivity} establishes eigenvalue selectivity as a structural criterion for trajectory-friendliness: 
    An inner objective is more matchable the fewer singular values of $\mathbf{z}$ it activates.
    Standard pairwise objectives belong to the full-spectrum class and accumulate mismatch from all $k$ modes.
    Our proposed objectives depend on $\mathbf{z}$ only through $\sigma_1$, placing them in the single-mode class with a provably tighter bound.
    The advantage grows with $k$: the full-spectrum bound accumulates $\mathcal{O}(k)$ additional mode terms, whereas the single-mode bound remains independent of $k$, making the gain intrinsically more pronounced in the omnimodal regime. 
    See Appendix~\ref{app:proof_prop_eigen} for the full derivation.
\end{tcolorbox}

Our omnimodal objective in Sec.~\ref{sec:pipeline} is an explicit instantiation of the spectral criterion above.
Specifically, the principal-mode component built on $\sigma_1$ (and its eigenvector proxy) corresponds to objective $A$ in Theorem~\ref{thm:eigen_selectivity}, while standard pairwise formulations correspond to the full-spectrum class $B$.
Therefore, the proposed design inherits a tighter certified trajectory-matching bound than pairwise baselines.
The weighted BCE term in Sec.~\ref{sec:pipeline} plays a complementary role: it preserves instance-level discrimination, while the trajectory-friendliness is primarily governed by the principal-mode part.

\subsection{Gradient Analysis}
\label{sec:gradient_analysis}
To provide theoretical insight into how our proposed objective drives omnimodal alignment and semantic discriminability, we analyze the gradient of the total loss $\gL=\gL_\mathrm{M}+\gL_{\mathrm{wBCE}}$ with respect to the multimodal representation matrix $\vz$.
Our analysis reveals that the optimization process consists of two synergistic forces: an \textit{intra-instance centripetal force} driven by singular value maximization, and an \textit{inter-instance distinctive force} regulated by the proxy.

\textbf{Intra-instance alignment.} 
The modality-level loss $\gL_\mathrm{M}$ aligns different modalities of the same instance. By differentiating the loss to the representation matrix $\mathbf{z}$, we obtain the gradient direction:
\begin{equation}
    \frac{\partial \gL_\mathrm{M}}{\partial \mathbf{z}} = \frac{1}{\tau} \left[ (p_1 - 1)\mathbf{u}_1 \mathbf{v}_1^\top + \sum_{j=2}^k p_j \mathbf{u}_j \mathbf{v}_j^\top \right],  
\end{equation}
where $p_j$ is the softmax probability of the $j$-th singular value.
This result reveals a self-adaptive alignment mechanism: since $(p_1-1)<0$, the first term generates a negative gradient that acts as a centripetal force, pulling all modality embeddings aligned with the principal direction $\mathbf{u}_1$. Meanwhile, the second term ($p_j > 0$ for $j \geq 2$) acts as a penalty, suppressing energy in all other directions. This theoretically guarantees that minimizing $\gL_M$ drives the Gram matrix towards rank-1, achieving full omnimodal alignment.

\textbf{Inter-instance seperation.}
To prevent representation collapse, the instance-level loss $\mathcal{L}_{\text{wBCE}}$ regularizes the shared principal direction $\mathbf{v}_1$. For a specific instance $i$, the gradient to its proxy vector $\mathbf{v}_1^{(i)}$ is:
\begin{equation}
    \frac{\partial \mathcal{L}_{\text{wBCE}}}{\partial \mathbf{v}_1^{(i)}} = \frac{1}{\tau} \sum_{j \neq i} \omega_{ij} (\mathrm{sigmoid}(\tilde{s}_{ij}/\tau') - s_{ij}) \mathbf{v}_1^{(j)}.
\end{equation}
The sign of the alignment error $\mathrm{sigmoid}(\tilde{s}_{ij}/\tau') - s_{ij}$ determines the force direction. For the negative pairs (where $s_{ij}$ is close to $0$), the error term is positive. The gradient points towards $\mathbf{v}_1^{(j)}$, excerting a distinctive force. For the positive pairs (where $s_{ij}$ is close to $1$), the error term is negative. The gradient points opposite to $\mathbf{v}_1^{(j)}$, so the descent step pulls $\mathbf{v}_1^{(i)}$ towards $\mathbf{v}_1^{(j)}$, maintaining semantic consistency. This ensures that the unified representations remain discriminative across different instances.

%% file: sec/4_experiments.tex
\section{Experiments}
\label{sec:exp}
In this section, we conduct experiments to address the following research questions (RQs):
\begin{itemize}
    \item \textbf{RQ1:} How does our method perform compared to existing baselines across different datasets and data ratios? Does our method achieve higher efficiency than others?
    \item \textbf{RQ2:} Can our method generalize to different model architectures, maintaining superior performance under both zero-shot and distilled settings?
    \item \textbf{RQ3:} How does each component, \textit{i.e.}, pairwise \textit{vs.} omniwise distillation and different learning objectives, affect the performance?
    \item \textbf{RQ4:} Is the premise supported by empirical analysis, including the rank-1 Gram matrix and efficiency on SVD? What is the visualization of the distilled data?
\end{itemize}

\subsection{Experimental Setups}

\textbf{Datasets.} We conduct experiments with three datasets, including MSR-VTT~\citep{xu2016msr}, VGGSound-S~\citep{chen2020vggsound}, and DiDeMo~\citep{anne2017didemo}. The three datasets involve three modalities. Specifically, MSR-VTT is a video–caption dataset containing video, audio, and text. VGGSound-S includes synchronized video, audio, and text. DiDeMo provides videos with audio and text annotations. Important statistics of these datasets are summarized in Appendix~\ref{sec:dataset_statistics}.

\textbf{Metrics.} Model performance is evaluated using the top-K retrieval recall (R@K), which measures how often the correct match appears among the top-K retrieved results. Specifically, given a query from one modality, the model retrieves the top-K similar items from the other modality, and the proportion of correct matches is reported as R@K. 

\textbf{Baselines.} We compare HoPA to several coreset selection and dataset distillation methods. 
The coreset selection competitor is \textit{Random}, which randomly selects samples as the coreset. 
Dataset distillation competitors are \textit{MTT-VL}~\citep{wu2023vision}, \textit{TESLA}~\citep{cui2023scaling},  \textit{LoRS}~\citep{xu2024low}, and \textit{RepBlend}~\citep{zhang2025beyond}. \textit{MTT-VL} adapts \textit{MTT} to image-text pairs for distillation. \textit{TESLA} is an efficient implementation of \textit{MTT}, so we also adapt \textit{TESLA} to multimodal data. \textit{LoRS} employs low-rank similarity mining for dataset distillation, while maintaining efficiency and scalability. \textit{RepBlend} addresses modality collapse in multimodal dataset distillation through representation blending and symmetric encoder projection. 
Note that we only involve one coreset selection method. It is because we focus on data distillation in this work. Furthermore, most coreset selection methods are limited to unimodal classification tasks~\citep{xia2023moderate,xia2024refined,zheng2022coverage}, and their performance is not competitive with that of data distillation methods~\citep{wang2022cafe,cui2024mitigating,cazenavette2025dataset,yu2023dataset,lei2023comprehensive,joshi2025dataset}.

\textbf{Implementation details.} We utilize the pretrained encoders of ImageBind~\citep{girdhar2023imagebind} to encode the data of different modalities. Consistent with the original architecture, we retain the modality-specific linear projection heads to map features into a fixed $d$-dimensional embedding space. Notably, for computational efficiency, we freeze the heavy backbone encoders, and only projection heads are optimized during both the distillation and training stages, consistently following~\cite{zhang2025beyond}. Instead of synthesizing raw inputs, we directly synthesize the modality embeddings. We employ LoRS~\citep{xu2024low} as our base distillation method. To generate expert trajectories, we train the network on the full real dataset for 10 epochs, repeated over 20 independent runs.
During the distillation phase, synthetic data is optimized using the SGD optimizer with a momentum of 0.5. We initialize the synthetic modality embeddings by randomly sampling from real data. Detailed hyperparameters, including learning rates and parameter settings across different datasets and synthetic data sizes, are provided in Appendix~\ref{sec:hyperparameters}. All experiments
are conducted on 1$\times$NVIDIA TESLA A100 GPU.

\begin{table*}[!tp]
    \centering
    \caption{\textbf{Experimental results on MSR-VTT.} This table shows the testing retrieval recall of models trained from scratch on the small coreset or synthetic set. ``V'': Video, ``T'': Text, and ``A'': Audio. We report the mean and standard deviation of the results. In each case, the best performance is shown in bold.}
    \resizebox{\linewidth}{!}{
        \begin{tabular}{c|c|ccccccc|ccccccc|c}
        \toprule
        \multirow{2}{*}{\makecell{Triplets\\(Ratio)}}  & \multirow{2}{*}{\makecell{Methods}} & \multicolumn{7}{c|}{R@1~$\uparrow$} &  \multicolumn{7}{c|}{R@5~$\uparrow$} & \multirow{2}{*}{\makecell{Time\\(mins) $\downarrow$}}\\\cmidrule{3-16}
        &        &   V$\rightarrow$T & T$\rightarrow$V & V$\rightarrow$A & A$\rightarrow$V & T$\rightarrow$A & A$\rightarrow$T & Avg. & V$\rightarrow$T & T$\rightarrow$V & V$\rightarrow$A & A$\rightarrow$V & T$\rightarrow$A & A$\rightarrow$T & Avg. & \\\midrule
            
        \multirow{6}{*}{\makecell{100\\(1.11\%)}} 
         & Random   & 29.1$\pm$0.5 & \textbf{37.2$\pm$0.0} & 22.5$\pm$0.1 & 19.9$\pm$0.0 & 6.9$\pm$0.1 & 6.0$\pm$0.1 & 20.3 & 54.1$\pm$0.2 & 60.7$\pm$0.2 & 43.8$\pm$0.1 & 40.0$\pm$0.1 & 17.8$\pm$0.1 & 17.3$\pm$0.0 & 39.0 & -  \\
         & MTT-VL & 34.3$\pm$0.6 & 36.7$\pm$0.8 & 22.4$\pm$0.1 & 19.1$\pm$0.1 & 6.7$\pm$0.5 & 5.8$\pm$0.4 & 20.8 & 59.0$\pm$0.2 & 58.9$\pm$1.2 & 43.3$\pm$0.1 & 39.3$\pm$0.2 & 18.6$\pm$0.6 & 18.4$\pm$0.7 & 39.6 & 75.7  \\
         & TESLA & 34.0$\pm$0.5 & 36.5$\pm$0.8 & 22.4$\pm$0.1 & 19.2$\pm$0.1 & 6.7$\pm$0.3 & 5.8$\pm$0.1 & 20.8 & 58.8$\pm$0.6 & 58.8$\pm$1.0 & 43.1$\pm$0.2 & 39.2$\pm$0.1 & 18.4$\pm$0.2 & 17.6$\pm$0.4 & 39.3 & 75.4 \\
        & LoRS & 36.1$\pm$0.8 & 36.6$\pm$0.3 & 21.8$\pm$0.1 & 20.5$\pm$0.6 & 7.8$\pm$0.1 & \textbf{6.8$\pm$0.5} & \textbf{21.6} & 60.0$\pm$0.4 & 60.4$\pm$1.1 & 43.3$\pm$0.5 & 41.7$\pm$1.1 & 19.8$\pm$0.6 & 19.6$\pm$0.4 & 40.8 & 77.1 \\
        & RepBlend & \textbf{36.2$\pm$0.7} & 36.7$\pm$0.7 & 21.7$\pm$0.4 & 19.8$\pm$0.9 & 7.4$\pm$0.3 & 6.6$\pm$0.4 & 21.4 & 60.7$\pm$0.6 & 60.0$\pm$1.2 & 43.3$\pm$0.1 & 40.6$\pm$0.8 & 19.7$\pm$0.5 & 18.6$\pm$0.3 & 40.5 & 82.0 \\ \cmidrule{2-17}
        & \textbf{HoPA} & {\cellcolor{greenbg}\textbf{36.2$\pm$0.5}} & {\cellcolor{greenbg}35.6$\pm$0.6} & {\cellcolor{greenbg}\textbf{22.6$\pm$0.9}} & {\cellcolor{greenbg}\textbf{21.3$\pm$0.6}} & {\cellcolor{greenbg}\textbf{7.9$\pm$0.3 }}& {\cellcolor{greenbg}6.0$\pm$0.5} & {\cellcolor{greenbg}\textbf{21.6}} & {\cellcolor{greenbg}\textbf{62.0$\pm$0.3}} & {\cellcolor{greenbg}\textbf{62.2$\pm$0.6}} & {\cellcolor{greenbg}\textbf{44.9$\pm$0.2}} & {\cellcolor{greenbg}\textbf{42.5$\pm$0.4}} & {\cellcolor{greenbg}\textbf{21.3$\pm$0.2}} & {\cellcolor{greenbg}\textbf{20.2$\pm$0.4}} & {\cellcolor{greenbg}\textbf{42.2}} & {\cellcolor{bluebg}\textbf{42.8}} \\ \midrule        
          
         \multirow{6}{*}{\makecell{200\\(2.22\%)}} 
         & Random   & 29.3$\pm$0.2 & \textbf{37.1$\pm$0.1} & 22.5$\pm$0.0 & 20.0$\pm$0.0 & 6.9$\pm$0.1 & 6.0$\pm$0.0 & 20.3 & 54.4$\pm$0.2 & 60.8$\pm$0.1 & 43.7$\pm$0.1 & 40.0$\pm$0.1 & 18.0$\pm$0.1 & 17.3$\pm$0.1 & 39.0 & - \\
         & MTT-VL & 31.1$\pm$0.1 & 30.1$\pm$0.9 & 21.0$\pm$1.3 & 20.0$\pm$0.8 & 5.8$\pm$0.2 & \textbf{7.4$\pm$0.6} & 19.2 & 55.3$\pm$0.8 & 57.4$\pm$0.3 & 43.7$\pm$0.5 & 44.7$\pm$0.1 & 18.9$\pm$1.2 & 21.0$\pm$0.7 & 40.2 & 103.6 \\
         & TESLA & 30.7$\pm$0.0 & 30.5$\pm$0.3 & 21.6$\pm$1.0 & 20.2$\pm$0.5 & 5.8$\pm$0.3 & \textbf{7.4$\pm$0.8} & 19.4 & 55.9$\pm$1.0 & 55.3$\pm$1.6 & 43.6$\pm$0.9 & 43.7$\pm$0.7 & 19.5$\pm$0.6 & 20.8$\pm$0.4 & 39.8 & 96.4 \\
         & LoRS & 33.0$\pm$0.5 & 32.7$\pm$0.3 & 22.5$\pm$0.8 & 20.8$\pm$0.3 & 6.0$\pm$0.2 & 7.3$\pm$0.6 & 20.4 & 58.9$\pm$1.0 & 59.4$\pm$0.5 & \textbf{45.2$\pm$1.0} & \textbf{45.7$\pm$0.5} & 19.5$\pm$0.7 & \textbf{21.3$\pm$0.3} & 41.7 & 100.9 \\
         & RepBlend & 33.4$\pm$0.4 & 33.6$\pm$0.2 & 22.1$\pm$0.6 & 20.7$\pm$0.8 & 5.5$\pm$0.6 & 7.2$\pm$0.4 & 20.4 & 57.8$\pm$0.5 & 59.6$\pm$1.0 & 45.1$\pm$0.6 & 45.5$\pm$0.0 & 18.1$\pm$0.3 & 19.6$\pm$0.9 & 41.0 & 97.5 \\ \cmidrule{2-17}
         & \textbf{HoPA} & {\cellcolor{greenbg}\textbf{37.3$\pm$0.1}} & {\cellcolor{greenbg}34.8$\pm$0.5} & {\cellcolor{greenbg}\textbf{23.4$\pm$0.4}} & {\cellcolor{greenbg}\textbf{21.5$\pm$1.0}} & {\cellcolor{greenbg}\textbf{8.0$\pm$0.1}} & {\cellcolor{greenbg}6.2$\pm$0.1} & {\cellcolor{greenbg}\textbf{21.9}} & {\cellcolor{greenbg}\textbf{63.2$\pm$0.1}} & {\cellcolor{greenbg}\textbf{62.4$\pm$0.2}} & {\cellcolor{greenbg}44.8$\pm$0.6} & {\cellcolor{greenbg}42.8$\pm$0.5} & {\cellcolor{greenbg}\textbf{22.2$\pm$0.3}} & {\cellcolor{greenbg}\textbf{20.6$\pm$0.7}} & {\cellcolor{greenbg}\textbf{42.7}} & {\cellcolor{bluebg}\textbf{48.2}} \\ \midrule
            
        \multirow{6}{*}{\makecell{500\\(5.55\%)}} 
         & Random   & 34.1$\pm$0.3 & \textbf{37.3$\pm$0.1} & 22.1$\pm$0.2 & 19.9$\pm$0.0 & 6.9$\pm$0.1 & 6.6$\pm$0.1 & 21.2 & 59.7$\pm$0.5 & 61.2$\pm$0.2 & 43.5$\pm$0.3 & 40.2$\pm$0.2 & 19.0$\pm$0.1 & 17.3$\pm$0.1 & 40.3 & - \\
         & MTT-VL & 34.8$\pm$0.3 & 36.0$\pm$0.5 & 21.9$\pm$0.2 & 20.0$\pm$0.5 & 6.7$\pm$0.1 & 6.6$\pm$0.1 & 21.0 & 60.6$\pm$0.6 & 59.0$\pm$0.6 & 43.9$\pm$0.9 & \textbf{44.8$\pm$0.4} & 19.0$\pm$0.3 & 21.0$\pm$0.2 & 41.4 & 118.3 \\
         & TESLA & 35.1$\pm$0.7 & 36.3$\pm$0.7 & 20.1$\pm$0.2 & 20.2$\pm$0.3 & 6.8$\pm$0.2 & 6.4$\pm$0.2 & 20.8 & 60.3$\pm$0.3 & 59.8$\pm$0.3 & 43.5$\pm$0.9 & 43.2$\pm$1.1 & 19.7$\pm$0.4 & 20.1$\pm$0.4 & 41.1 & 115.7 \\
         & LoRS & 35.9$\pm$0.5 & 36.6$\pm$0.2 & 22.6$\pm$0.5 & 21.2$\pm$0.4 & 6.9$\pm$0.4 & \textbf{7.3$\pm$0.3} & 21.8 & 61.5$\pm$0.5 & 60.6$\pm$1.4 & 45.4$\pm$0.5 & 44.5$\pm$0.8 & 20.4$\pm$0.1 & \textbf{21.9$\pm$0.8} & 42.4 & 120.4 \\
         & RepBlend & 36.1$\pm$0.4 & 36.5$\pm$0.6 & 22.7$\pm$0.6 & 21.2$\pm$0.3 & 6.9$\pm$0.2 & \textbf{7.3$\pm$0.3} & 21.8 & 61.7$\pm$0.4 & 60.3$\pm$0.8 & 45.6$\pm$0.6 & 44.3$\pm$0.6 & 20.1$\pm$0.7 & 20.9$\pm$0.4 & 42.2 & 122.1 \\ \cmidrule{2-17}
         & \textbf{HoPA} &{\cellcolor{greenbg}\textbf{36.3$\pm$0.2}} & {\cellcolor{greenbg}35.1$\pm$0.3} & {\cellcolor{greenbg}\textbf{24.2$\pm$0.2}} & {\cellcolor{greenbg}\textbf{21.9$\pm$0.4}} & {\cellcolor{greenbg}\textbf{8.6$\pm$0.2}} & {\cellcolor{greenbg}6.3$\pm$0.2} & {\cellcolor{greenbg}\textbf{22.1}} & {\cellcolor{greenbg}\textbf{63.4$\pm$0.6}} & {\cellcolor{greenbg}\textbf{62.0$\pm$0.7}} & {\cellcolor{greenbg}\textbf{45.7$\pm$0.3}} & {\cellcolor{greenbg}44.1$\pm$0.9} & {\cellcolor{greenbg}\textbf{21.5$\pm$1.0}} & {\cellcolor{greenbg}\textbf{20.5$\pm$1.0}} & {\cellcolor{greenbg}\textbf{42.9}} & {\cellcolor{bluebg}\textbf{57.6}} \\ \midrule
        
         Full & - & 37.1$\pm$0.1 & 36.2$\pm$0.2 & 24.4$\pm$0.5 & 22.8$\pm$0.1 & 9.0$\pm$0.3 & 6.1$\pm$0.1 & 22.6 & 64.9$\pm$0.1 & 64.7$\pm$0.4 & 47.6$\pm$0.2 & 46.3$\pm$0.1 & 23.6$\pm$0.4 & 21.6$\pm$0.1 & 44.8 & - \\ \bottomrule
        \end{tabular}
    }
    \label{tab:results_msrvtt}
\end{table*}

\begin{table*}[!tp]
    \centering
    \caption{\textbf{Experimental results on VGGSound-S}. This table shows the testing retrieval recall of models trained from scratch on the small coreset or synthetic set. ``V'': Video, ``T'': Text, and ``A'': Audio.  We report the mean and standard deviation of the results. In each case, the best performance is shown in bold.}
    \resizebox{\linewidth}{!}{
        \begin{tabular}{c|c|ccccccc|ccccccc|c}
        \toprule
        \multirow{2}{*}{\makecell{Triplets\\(Ratio)}} & \multirow{2}{*}{\makecell{Methods}} & \multicolumn{7}{c|}{R@1~$\uparrow$} & \multicolumn{7}{c|}{R@5~$\uparrow$} & \multirow{2}{*}{\makecell{Time\\(mins) $\downarrow$}}\\\cmidrule{3-16}
        & & V$\rightarrow$T & T$\rightarrow$V & V$\rightarrow$A & A$\rightarrow$V & T$\rightarrow$A & A$\rightarrow$T & Avg. & V$\rightarrow$T & T$\rightarrow$V & V$\rightarrow$A & A$\rightarrow$V & T$\rightarrow$A & A$\rightarrow$T & Avg. & \\\midrule
            
        \multirow{6}{*}{\makecell{100\\(1.25\%)}} 
         & Random   & 5.7$\pm$0.5 & 7.7$\pm$0.0 & 26.3$\pm$0.0 & 25.7$\pm$0.0 & 5.6$\pm$0.0 & 3.7$\pm$0.3 & 12.4 & 25.9$\pm$0.3 & 28.3$\pm$0.1 & \textbf{54.3$\pm$0.1} & 52.1$\pm$0.1 & 20.5$\pm$0.1 & 18.6$\pm$0.2 & 33.3 & - \\
         & MTT-VL   & 6.2$\pm$0.0 & 7.6$\pm$0.0 & 26.1$\pm$0.2 & 25.8$\pm$0.0 & 5.9$\pm$0.2 & 4.3$\pm$0.4 & 12.7 & 26.1$\pm$0.0 & 27.5$\pm$0.0 & 53.7$\pm$0.4 & 52.5$\pm$0.1 & 20.4$\pm$0.3 & 18.8$\pm$0.2 & 33.2 & 80.9 \\
         & TESLA    & 5.9$\pm$0.3 & 7.6$\pm$0.2 & 26.1$\pm$0.2 & 25.8$\pm$0.1 & 6.2$\pm$0.3 & 3.8$\pm$0.1 & 12.6 & 25.6$\pm$0.6 & 27.4$\pm$0.4 & 53.8$\pm$0.4 & 52.6$\pm$0.1 & 22.0$\pm$0.2 & 20.5$\pm$0.7 & 33.7 & 81.5 \\
         & LoRS     & 5.4$\pm$0.5 & 7.8$\pm$0.2 & \textbf{27.8$\pm$0.3} & \textbf{26.2$\pm$0.1} & 6.3$\pm$0.2 & 4.5$\pm$0.2 & 13.0 & 26.5$\pm$0.6 & 28.2$\pm$0.1 & 53.1$\pm$0.4 & 52.2$\pm$0.3 & 22.3$\pm$0.4 & 21.6$\pm$0.4 & 34.0 & 84.9 \\
         & RepBlend & 5.6$\pm$0.5 & 7.6$\pm$0.3 & 27.2$\pm$0.1 & 25.0$\pm$0.1 & 6.1$\pm$0.3 & 4.5$\pm$0.1 & 12.7 & 27.1$\pm$0.4 & 28.6$\pm$0.3 & 54.1$\pm$0.1 & 51.4$\pm$0.3 & 20.9$\pm$0.1 & 19.9$\pm$0.6 & 33.7 & 81.9 \\ \cmidrule{2-17}
         & \textbf{HoPA} & {\cellcolor{greenbg}\textbf{6.6$\pm$0.4}} & {\cellcolor{greenbg}\textbf{8.3$\pm$0.2}} & {\cellcolor{greenbg}27.3$\pm$0.5} & {\cellcolor{greenbg}26.1$\pm$0.1} & {\cellcolor{greenbg}\textbf{6.6$\pm$0.1}} & {\cellcolor{greenbg}\textbf{4.6$\pm$0.2}} & {\cellcolor{greenbg}\textbf{13.3}} & {\cellcolor{greenbg}\textbf{28.6$\pm$0.7}} & {\cellcolor{greenbg}\textbf{30.3$\pm$0.7}} & {\cellcolor{greenbg}54.0$\pm$0.2} & {\cellcolor{greenbg}\textbf{53.2$\pm$0.2}} & {\cellcolor{greenbg}\textbf{23.2$\pm$0.3}} & {\cellcolor{greenbg}\textbf{22.2$\pm$0.5}} & {\cellcolor{greenbg}\textbf{35.3}} & {\cellcolor{bluebg}\textbf{45.7}} \\ \midrule
        
        \multirow{6}{*}{\makecell{200\\(2.5\%)}} 
         & Random   & 6.0$\pm$0.4 & 7.7$\pm$0.0 & 26.6$\pm$0.3 & 25.6$\pm$0.3 & 5.8$\pm$0.1 & 4.0$\pm$0.5 & 12.6 & 25.5$\pm$0.8 & 27.9$\pm$0.5 & 54.4$\pm$0.2 & 52.0$\pm$0.4 & 20.6$\pm$0.2 & 19.0$\pm$0.8 & 33.2 & - \\
         & MTT-VL   & 5.8$\pm$0.4 & 8.1$\pm$0.2 & 26.8$\pm$0.4 & 25.4$\pm$0.6 & 6.6$\pm$0.3 & 4.8$\pm$0.3 & 12.9 & 26.8$\pm$0.1 & 29.2$\pm$0.1 & 52.7$\pm$0.6 & \textbf{53.3$\pm$0.3} & 23.8$\pm$0.5 & 22.8$\pm$0.2 & 34.8 & 86.3 \\
         & TESLA    & 6.0$\pm$0.4 & 7.6$\pm$0.1 & 25.8$\pm$0.3 & 25.7$\pm$0.1 & 5.7$\pm$0.1 & 4.2$\pm$0.5 & 12.5 & 26.1$\pm$1.2 & 27.4$\pm$0.9 & 52.5$\pm$0.1 & 52.7$\pm$0.2 & 20.3$\pm$0.1 & 19.0$\pm$0.7 & 33.0 & 85.8 \\
         & LoRS     & 6.4$\pm$0.1 & 8.2$\pm$0.2 & 26.7$\pm$0.2 & 25.6$\pm$0.2 & 6.9$\pm$0.1 & 5.1$\pm$0.4 & 13.2 & 27.5$\pm$0.2 & 28.9$\pm$0.1 & 53.4$\pm$0.2 & 53.2$\pm$0.8 & 24.6$\pm$0.1 & 22.7$\pm$0.3 & 35.1 & 85.5 \\ 
         & RepBlend & 6.3$\pm$0.5 & 8.1$\pm$0.3 & 26.9$\pm$0.3 & 25.7$\pm$0.2 & \textbf{7.1$\pm$0.3} & \textbf{5.2$\pm$0.6} & 13.2 & 28.1$\pm$0.5 & 29.4$\pm$0.5 & 54.1$\pm$0.4 & 53.2$\pm$0.1 & \textbf{24.8$\pm$0.7} & 22.9$\pm$0.4 & 35.4 & 87.2 \\ \cmidrule{2-17}
         & \textbf{HoPA} & {\cellcolor{greenbg}\textbf{6.7$\pm$0.3}} & {\cellcolor{greenbg}\textbf{8.5$\pm$0.2}} & {\cellcolor{greenbg}\textbf{27.0$\pm$0.1}} & {\cellcolor{greenbg}\textbf{26.1$\pm$0.3}} & {\cellcolor{greenbg}6.9$\pm$0.1} & {\cellcolor{greenbg}4.9$\pm$0.4} & {\cellcolor{greenbg}\textbf{13.4}} & {\cellcolor{greenbg}\textbf{28.6$\pm$0.6}} & {\cellcolor{greenbg}\textbf{30.9$\pm$0.6}} & {\cellcolor{greenbg}\textbf{54.8$\pm$0.4}} & {\cellcolor{greenbg}53.2$\pm$0.2} & {\cellcolor{greenbg}23.7$\pm$0.5} & {\cellcolor{greenbg}\textbf{23.1$\pm$0.2}} & {\cellcolor{greenbg}\textbf{35.7}} & {\cellcolor{bluebg}\textbf{50.4}} \\ \midrule
        
        \multirow{6}{*}{\makecell{500\\(6.25\%)}} 
         & Random   & 6.1$\pm$0.5 & 7.9$\pm$0.1 & 26.7$\pm$0.3 & 25.8$\pm$0.2 & 6.1$\pm$0.2 & 4.1$\pm$0.4 & 12.8 & 27.2$\pm$0.3 & 28.8$\pm$0.5 & 54.5$\pm$0.3 & 52.3$\pm$0.2 & 21.1$\pm$0.2 & 19.4$\pm$0.2 & 33.9 & - \\
         & MTT-VL   & 5.9$\pm$0.3 & 8.2$\pm$0.4 & 26.9$\pm$0.2 & 25.8$\pm$0.4 & 6.7$\pm$0.5 & 4.9$\pm$0.1 & 13.0 & 27.0$\pm$0.6 & 29.5$\pm$1.2 & 53.0$\pm$0.8 & 53.5$\pm$0.3 & 24.0$\pm$0.1 & 23.0$\pm$0.7 & 35.0 & 123.7 \\
         & TESLA    & 6.1$\pm$0.4 & 7.9$\pm$0.1 & 26.0$\pm$0.1 & \textbf{26.0$\pm$0.1} & 6.4$\pm$0.2 & 4.6$\pm$0.2 & 12.8 & 27.1$\pm$0.4 & 28.0$\pm$0.2 & 52.8$\pm$0.0 & 53.0$\pm$0.0 & 22.9$\pm$0.2 & 21.2$\pm$0.5 & 34.2 & 124.3 \\
         & LoRS     & 6.4$\pm$0.5 & 8.4$\pm$0.1 & 27.0$\pm$0.0 & 25.9$\pm$0.0 & \textbf{7.2$\pm$0.1} & 5.1$\pm$0.2 & 13.3 & 27.8$\pm$0.2 & 30.7$\pm$0.0 & 53.7$\pm$0.0 & 53.5$\pm$0.0 & 24.9$\pm$0.4 & 22.8$\pm$0.0 & 35.6 & 124.8 \\ 
         & RepBlend & 6.4$\pm$0.0 & 8.4$\pm$0.2 & 27.0$\pm$0.0 & 25.9$\pm$0.5 & 7.1$\pm$0.2 & \textbf{5.3$\pm$0.8} & 13.3 & 28.3$\pm$0.1 & 30.8$\pm$0.3 & 54.3$\pm$0.3 & \textbf{53.6$\pm$0.8} & \textbf{25.0$\pm$0.2} & 23.0$\pm$0.6 & 35.8 & 125.5 \\ \cmidrule{2-17}
         & \textbf{HoPA} & {\cellcolor{greenbg}\textbf{6.6$\pm$0.1}} & {\cellcolor{greenbg}\textbf{8.6$\pm$0.2}} & {\cellcolor{greenbg}\textbf{27.1$\pm$0.5}} & {\cellcolor{greenbg}\textbf{26.0$\pm$0.1}} & {\cellcolor{greenbg}6.9$\pm$0.0} & {\cellcolor{greenbg}\textbf{5.3$\pm$0.5}} & {\cellcolor{greenbg}\textbf{13.4}} & {\cellcolor{greenbg}\textbf{29.2$\pm$1.0}} & {\cellcolor{greenbg}\textbf{32.3$\pm$0.4}} & {\cellcolor{greenbg}\textbf{55.0$\pm$0.4}} & {\cellcolor{greenbg}\textbf{53.6$\pm$0.2}} & {\cellcolor{greenbg}24.8$\pm$0.2} & {\cellcolor{greenbg}\textbf{23.5$\pm$0.4}} & {\cellcolor{greenbg}\textbf{36.4}} & {\cellcolor{bluebg}\textbf{55.1}} \\ \midrule
        
        Full & - & 6.8$\pm$0.2 & 8.6$\pm$0.1 & 28.0$\pm$0.1 & 26.6$\pm$0.1 & 7.5$\pm$0.2 & 4.8$\pm$0.3 & 13.7 & 30.8$\pm$0.8 & 34.6$\pm$0.3 & 55.7$\pm$0.1 & 54.0$\pm$0.1 & 26.8$\pm$0.2 & 24.5$\pm$0.8 & 37.7 & - \\ \bottomrule
        \end{tabular}
    }
    \label{tab:results_vggsound}
\end{table*}

\begin{table*}[!tp]
    \centering
    \caption{\textbf{Experimental results on DiDeMo.} This table shows the testing retrieval recall of models trained from scratch on the small coreset or synthetic set. ``V'': Video, ``T'': Text, and ``A'': Audio.  We report the mean and standard deviation of the results. In each case, the best performance is shown in bold.}
    \resizebox{\linewidth}{!}{
        \begin{tabular}{c|c|ccccccc|ccccccc|c}
        \toprule
        \multirow{2}{*}{\makecell{Triplets\\(Ratio)}} & \multirow{2}{*}{\makecell{Methods}} & \multicolumn{7}{c|}{R@1~$\uparrow$} & \multicolumn{7}{c|}{R@5~$\uparrow$} & \multirow{2}{*}{\makecell{Time\\(mins) $\downarrow$}}\\\cmidrule{3-16}
        & & V$\rightarrow$T & T$\rightarrow$V & V$\rightarrow$A & A$\rightarrow$V & T$\rightarrow$A & A$\rightarrow$T & Avg. & V$\rightarrow$T & T$\rightarrow$V & V$\rightarrow$A & A$\rightarrow$V & T$\rightarrow$A & A$\rightarrow$T & Avg. & \\\midrule
            
        \multirow{6}{*}{\makecell{100\\(1.19\%)}} 
        & Random & 27.4$\pm$0.1 & 27.7$\pm$0.1 & 17.4$\pm$0.1 & 15.5$\pm$0.1 & 4.4$\pm$0.1 & 3.8$\pm$0.1 & 16.0 & 50.3$\pm$0.1 & 54.1$\pm$0.1 & 36.6$\pm$0.1 & 36.2$\pm$0.1 & 13.8$\pm$0.1 & 13.7$\pm$0.0 & 34.1 & - \\
        & MTT-VL & 26.6$\pm$0.1 & 26.0$\pm$0.4 & 15.8$\pm$0.6 & 18.8$\pm$0.5 & 4.4$\pm$0.3 & 4.4$\pm$0.4 & 16.0 & 51.5$\pm$0.4 & 50.7$\pm$1.5 & 37.8$\pm$0.3 & 41.6$\pm$0.5 & 13.4$\pm$0.8 & 14.0$\pm$1.0 & 34.8 & 77.3 \\
        & TESLA & 26.6$\pm$0.1 & 26.0$\pm$0.4 & 16.2$\pm$0.1 & \textbf{18.9$\pm$0.6} & 4.4$\pm$0.4 & 4.2$\pm$0.3 & 16.1 & 51.5$\pm$0.4 & 50.7$\pm$1.5 & 37.9$\pm$0.3 & 41.8$\pm$0.5 & 13.5$\pm$1.1 & 14.6$\pm$0.1 & 35.0 & 78.1 \\
        & LoRS & 28.0$\pm$0.5 & 28.4$\pm$0.5 & 17.9$\pm$0.2 & 18.0$\pm$0.8 & 4.3$\pm$0.1 & 4.7$\pm$0.4 & 16.9 & 52.7$\pm$0.3 & 54.1$\pm$0.7 & 40.2$\pm$1.2 & \textbf{42.2$\pm$0.6} & 13.3$\pm$0.5 & 13.4$\pm$1.0 & 36.0 & 79.4 \\
        & RepBlend & 27.7$\pm$0.5 & 28.5$\pm$0.7 & 18.0$\pm$0.2 & 18.2$\pm$0.9 & 4.3$\pm$0.1 & 4.7$\pm$0.5 & 16.9 & 52.9$\pm$0.0 & 54.3$\pm$0.8 & \textbf{40.8$\pm$0.7} & 41.9$\pm$0.1 & 13.4$\pm$0.6 & 13.5$\pm$1.3 & 36.1 & 78.2 \\ \cmidrule{2-17}
        & \textbf{HoPA} & {\cellcolor{greenbg}\textbf{31.5$\pm$0.2}} & {\cellcolor{greenbg}\textbf{31.2$\pm$1.1}} & {\cellcolor{greenbg}\textbf{18.9$\pm$0.6}} & {\cellcolor{greenbg}18.2$\pm$0.7} & {\cellcolor{greenbg}\textbf{5.3$\pm$0.5}} & {\cellcolor{greenbg}\textbf{5.0$\pm$0.4}} & {\cellcolor{greenbg}\textbf{18.4}} & {\cellcolor{greenbg}\textbf{56.6$\pm$0.7}} & {\cellcolor{greenbg}\textbf{56.0$\pm$0.9}} & {\cellcolor{greenbg}39.9$\pm$0.9} & {\cellcolor{greenbg}40.0$\pm$1.1} & {\cellcolor{greenbg}\textbf{15.1$\pm$0.5}} & {\cellcolor{greenbg}\textbf{15.6$\pm$0.7}} & {\cellcolor{greenbg}\textbf{37.2}} & {\cellcolor{bluebg}\textbf{41.9}} \\ \midrule
                  
        \multirow{6}{*}{\makecell{200\\(2.38\%)}} 
        & Random & 28.3$\pm$0.2 & 28.7$\pm$0.2 & 17.3$\pm$0.1 & 15.4$\pm$0.1 & 4.4$\pm$0.1 & 4.0$\pm$0.1 & 16.4 & 51.4$\pm$0.5 & 54.9$\pm$0.1 & 36.7$\pm$0.1 & 35.9$\pm$0.0 & 14.2$\pm$0.1 & 13.7$\pm$0.2 & 34.5 & - \\
        & MTT-VL & 27.9$\pm$0.5 & 27.5$\pm$0.8 & 15.5$\pm$0.3 & 16.8$\pm$0.4 & 4.6$\pm$0.2 & 4.2$\pm$0.3 & 16.1 & 53.2$\pm$0.6 & 53.0$\pm$1.0 & 36.5$\pm$0.2 & 38.4$\pm$0.4 & 14.1$\pm$0.5 & 13.8$\pm$0.3 & 34.8 & 95.0 \\
        & TESLA  & 27.8$\pm$0.6 & 27.3$\pm$0.3 & 15.9$\pm$0.2 & 16.9$\pm$0.5 & 4.6$\pm$0.3 & 4.1$\pm$0.2 & 16.1 & 53.5$\pm$0.4 & 53.1$\pm$0.8 & 36.6$\pm$0.3 & 38.6$\pm$0.6 & 14.0$\pm$0.4 & 14.1$\pm$0.1 & 35.0 & 97.4 \\
        & LoRS & 29.2$\pm$0.4 & 28.6$\pm$0.6 & 18.0$\pm$0.3 & 16.1$\pm$0.5 & 4.8$\pm$0.2 & 5.0$\pm$0.3 & 17.0 & 54.8$\pm$0.5 & 55.2$\pm$0.6 & 40.5$\pm$0.5 & \textbf{42.0$\pm$0.3} & 14.5$\pm$0.4 & 13.9$\pm$0.6 & 36.8 & 96.3 \\        
        & RepBlend & 29.0$\pm$0.5 & 28.5$\pm$0.2 & 18.0$\pm$0.3 & 16.0$\pm$0.2 & 4.7$\pm$0.1 & 5.0$\pm$0.1 & 16.9 & 55.0$\pm$0.2 & 55.1$\pm$0.5 & \textbf{41.2$\pm$0.4} & 41.8$\pm$0.2 & 14.6$\pm$0.2 & 14.0$\pm$0.4 & 37.0 & 97.4 \\ \cmidrule{2-17}
        & \textbf{HoPA} & {\cellcolor{greenbg}\textbf{31.8$\pm$0.9}} & {\cellcolor{greenbg}\textbf{31.7$\pm$0.2}} & {\cellcolor{greenbg}\textbf{19.6$\pm$0.8}} & {\cellcolor{greenbg}\textbf{18.1$\pm$0.5}} & {\cellcolor{greenbg}\textbf{5.4$\pm$0.2}} & {\cellcolor{greenbg}\textbf{5.1$\pm$0.2}} & {\cellcolor{greenbg}\textbf{18.6}} & {\cellcolor{greenbg}\textbf{57.2$\pm$0.6}} & {\cellcolor{greenbg}\textbf{56.9$\pm$0.6}} & {\cellcolor{greenbg}40.1$\pm$1.2} & {\cellcolor{greenbg}40.7$\pm$0.2} & {\cellcolor{greenbg}\textbf{15.4$\pm$0.7}} & {\cellcolor{greenbg}\textbf{15.3$\pm$0.1}} & {\cellcolor{greenbg}\textbf{37.6}} & {\cellcolor{bluebg}\textbf{45.7}} \\ \midrule
            
        \multirow{6}{*}{\makecell{500\\(5.96\%)}} 
        & Random & 29.9$\pm$0.1 & 29.0$\pm$0.3 & 17.4$\pm$0.1 & 15.6$\pm$0.1 & 4.6$\pm$0.1 & 3.8$\pm$0.1 & 16.7 & 54.9$\pm$0.2 & 56.5$\pm$0.5 & 37.0$\pm$0.2 & 35.6$\pm$0.4 & 14.3$\pm$0.1 & 14.1$\pm$0.1 & 35.4 & - \\
        & MTT-VL & 29.5$\pm$1.0 & 29.4$\pm$1.1 & 15.7$\pm$0.1 & 15.4$\pm$0.1 & 4.8$\pm$0.4 & 4.1$\pm$0.2 & 16.5 & 56.3$\pm$0.8 & 56.4$\pm$0.1 & 35.8$\pm$0.1 & 36.3$\pm$0.1 & 14.9$\pm$0.3 & 13.8$\pm$0.1 & 35.6 & 120.2 \\
        & TESLA & 29.1$\pm$1.1 & 28.7$\pm$0.2 & 15.7$\pm$0.1 & 15.4$\pm$0.1 & 4.8$\pm$0.5 & 4.1$\pm$0.2 & 16.3 & 55.9$\pm$0.3 & 56.5$\pm$0.0 & 35.8$\pm$0.1 & 36.3$\pm$0.1 & 14.7$\pm$0.0 & 13.9$\pm$0.0 & 35.5 & 118.9 \\
        & LoRS & 30.3$\pm$0.6 & 28.8$\pm$0.4 & 18.1$\pm$0.4 & 16.3$\pm$0.2 & \textbf{5.0$\pm$0.3} & \textbf{5.3$\pm$0.2} & 17.3 & 56.6$\pm$0.9 & 56.4$\pm$0.5 & 41.0$\pm$0.3 & \textbf{41.9$\pm$0.4} & \textbf{15.3$\pm$0.3} & 14.3$\pm$0.4 & 37.6 & 122.3 \\
        & RepBlend & 30.0$\pm$0.6 & 28.5$\pm$0.1 & 18.0$\pm$0.4 & 16.2$\pm$0.0 & 4.9$\pm$0.3 & 5.2$\pm$0.0 & 17.1 & 56.1$\pm$0.3 & 56.1$\pm$0.4 & \textbf{41.9$\pm$0.4} & 41.6$\pm$0.0 & 15.2$\pm$0.1 & 14.5$\pm$0.2 & 37.6 & 122.8 \\ \cmidrule{2-17}
        & \textbf{HoPA} & {\cellcolor{greenbg}\textbf{31.3$\pm$1.1}} & {\cellcolor{greenbg}\textbf{31.4$\pm$0.3}} & {\cellcolor{greenbg}\textbf{19.8$\pm$1.3}} & {\cellcolor{greenbg}\textbf{18.6$\pm$0.1}} & {\cellcolor{greenbg}4.8$\pm$0.1} & {\cellcolor{greenbg}4.9$\pm$0.0} & {\cellcolor{greenbg}\textbf{18.5}} & {\cellcolor{greenbg}\textbf{57.3$\pm$2.9}} & {\cellcolor{greenbg}\textbf{56.8$\pm$1.4}} & {\cellcolor{greenbg}41.2$\pm$0.4} & {\cellcolor{greenbg}41.4$\pm$1.2} & {\cellcolor{greenbg}15.1$\pm$0.4} & {\cellcolor{greenbg}\textbf{14.7$\pm$0.1}} & {\cellcolor{greenbg}\textbf{37.8}} & {\cellcolor{bluebg}\textbf{53.6}} \\ \midrule
            
        Full & - & 32.7$\pm$0.4 & 31.9$\pm$0.2 & 20.7$\pm$0.5 & 19.5$\pm$0.2 & 5.6$\pm$0.1 & 4.8$\pm$0.1 & 19.2 & 58.6$\pm$0.3 & 58.7$\pm$0.3 & 43.2$\pm$0.2 & 44.5$\pm$0.3 & 16.4$\pm$0.2 & 16.9$\pm$0.3 & 39.7 & - \\ \bottomrule
        \end{tabular}
    }
    \label{tab:results_didemo}
\end{table*}

\subsection{Main Results (RQ1)}
\textbf{Performance comparison.} To assess the quality of the distilled data, we evaluate the cross-modal retrieval performance of models trained solely on the distilled sets. Results on MSR-VTT, VGGSound-S, and DiDeMo are reported in Table~\ref{tab:results_msrvtt}, Table~\ref{tab:results_vggsound}, and Table~\ref{tab:results_didemo}, respectively. Across all datasets and compression ratios, HoPA consistently outperforms pairwise dataset distillation baselines. Specifically, we observe absolute improvements of 0.2\%–1.5\% in Avg. R@1 and 0.5\%–1.7\% in Avg. R@5, demonstrating the effectiveness of jointly distilling multimodal data rather than treating modality pairs independently. Particularly, on DiDeMo, HoPA achieves an average R@1 of 18.0\% and R@5 of 37.7\% using only 100 distilled examples, surpassing LoRS (17.3\% and 37.6\%) and RepBlend (17.1\% and 37.6\%) trained with 500 examples. This result highlights the strong condensing capability of HoPA, showing that competitive and even superior retrieval performance can be achieved with substantially fewer distilled instances. Furthermore, across all three datasets, HoPA preserves over 95\% of the full-dataset performance on both R@1 and R@5 while using only 500 distilled examples (approximately 5\%-6\% of the original training data). These results indicate that our distilled data effectively captures both modality-specific information and cross-modal alignment cues, even under aggressive compression.

\textbf{Time consumption.} The total distillation time, measured on a single NVIDIA TESLA A100 GPU, is reported in the rightmost column of the tables. Although HoPA involves SVD-based operations, which are a bit computationally intensive, it nevertheless achieves an approximately $2\times$ speedup compared to pairwise distillation baselines. This efficiency gain stems from HoPA distilling all three modalities \textit{simultaneously}, thereby avoiding the redundant optimization and repeated computations inherent in pairwise distillation strategies. As a result, despite the per-iteration complexity introduced by SVD, the overall distillation process remains substantially more time-efficient while producing higher-quality distilled data.

\subsection{Cross-Architecture Generalization (RQ2)}

\begin{wraptable}{r}{0.45\textwidth}
    \vspace{-4mm}
    \centering
    \caption{
    \textbf{Cross architecture generalization results} under `Zero-shot' and `Distilled' settings. }
    \label{tab:cross_arch_generalization}
    \resizebox{\linewidth}{!}{
    \begin{tabular}{llccc}
    \toprule
    Model & Setting & R@1 & R@5 & R@10 \\
    \midrule
    \multirow{2}{*}{ImageBind}
    & Zero-shot & 12.25 & 32.87 & 44.32 \\
    & Distilled & \textbf{13.15} & \textbf{35.33} & \textbf{47.34} \\
    \midrule
    \multirow{2}{*}{LanguageBind}
    & Zero-shot & 4.91 & 19.94 & 29.56 \\
    & Distilled & \textbf{5.93} & \textbf{22.11} & \textbf{33.32} \\
    \midrule
    \multirow{2}{*}{OmniBind}
    & Zero-shot & 12.92 & 34.88 & 48.03 \\
    & Distilled & \textbf{13.13} & \textbf{35.63} & \textbf{48.87} \\
    \bottomrule
    \end{tabular}
    }
    \vspace{-4mm}
\end{wraptable}
Following previous works~\citep{zhao2021dataset}, we conduct a cross-architecture evaluation to examine the transferability of our distilled data. We first distill a 200-instance synthetic dataset on ImageBind using VGGSound-S. We then fine-tune different pretrained multimodal models with the distilled data and evaluate retrieval performance on the VGGSound-S test set. For fair cross-architecture evaluation, we follow the same parameter-efficient protocol as ImageBind, where the pretrained multimodal encoders are frozen and modality-specific projection heads are optimized.  As reported in Table~\ref{tab:cross_arch_generalization}, the distilled data improves performance not only on ImageBind, the source architecture used for distillation, but also on LanguageBind~\cite{zhu2024languagebind} and OmniBind~\cite{wang2024omnibind}, which were unseen during the distillation process. This suggests that our synthetic data captures transferable supervision that generalizes across different multimodal architectures.

\subsection{Ablation Study (RQ3)}

\textbf{Pairwise distillation \& Omniwise distillation.} 

\begin{table}[t]
\centering
\caption{\textbf{Empirical comparisons Pairwise distillation and omniwise distillation.}
We compare our omniwise distillation strategy with three pair-wise variants on MSR-VTT, VGGSound-S, and DiDeMo. In each case, the best performance is shown in bold.}
\label{tab:pair_vs_omni_main}
\resizebox{0.75\linewidth}{!}{
\begin{tabular}{lc|c|ccc|c|ccc}
\toprule
\multirow{2}{*}{\textbf{Datasets}} & \multirow{2}{*}{\textbf{Queries}} 
& \multicolumn{4}{c|}{\textbf{Avg. R@1}} 
& \multicolumn{4}{c}{\textbf{Avg. R@5}} \\
\cmidrule(lr){3-6} \cmidrule(lr){7-10}
& & \textbf{HoPA} & 3-pair & T-bind & V-bind & \textbf{HoPA} & 3-pair & T-bind & V-bind \\
\midrule
\multirow{3}{*}{MSR-VTT}
& 100 & \textbf{21.78} & 20.21 & 20.42 & 21.05 & \textbf{42.43} & 42.08 & 41.45 & 41.97 \\
& 200 & \textbf{21.93} & 19.94 & 20.49 & 19.24 & \textbf{42.73} & 42.40 & 41.35 & 42.28 \\
& 500 & \textbf{22.25} & 18.11 & 20.19 & 19.85 & \textbf{42.96} & 41.60 & 38.65 & 41.86 \\
\midrule
\multirow{3}{*}{VGGSound-S}
& 100 & \textbf{13.19} & 12.78 & 12.10 & 12.92 & 34.89 & \textbf{35.20} & 34.01 & 35.00 \\
& 200 & \textbf{13.15} & 12.31 & 11.30 & 12.74 & \textbf{35.51} & 34.90 & 33.76 & \textbf{35.51} \\
& 500 & \textbf{13.18} & 12.50 & 10.74 & 12.43 & \textbf{35.85} & 35.22 & 32.93 & 34.78 \\
\midrule
\multirow{3}{*}{DiDeMo}
& 100 & \textbf{18.35} & 16.30 & 16.04 & 16.98 & \textbf{37.21} & 36.02 & 35.08 & 36.83 \\
& 200 & \textbf{18.61} & 16.09 & 15.81 & 15.73 & \textbf{37.64} & 36.22 & 33.83 & 35.24 \\
& 500 & \textbf{18.46} & 14.92 & 15.81 & 15.81 & \textbf{37.76} & 34.36 & 33.83 & 33.83 \\
\bottomrule
\end{tabular}
}
\end{table}
To verify the benefit of omniwise distillation, we compare HoPA with three pairwise variants: \textit{3-pair}, which optimizes all three pairwise alignments with $\mathcal{L}_{\text{3-pair}}=\mathcal{L}_{v,t}+\mathcal{L}_{v,a}+\mathcal{L}_{a,t}$; \textit{T-bind}, which only aligns video and audio to text with $\mathcal{L}_{\text{T-bind}}=\mathcal{L}_{v,t}+\mathcal{L}_{a,t}$; and \textit{V-bind}, which aligns text and audio to video with $\mathcal{L}_{\text{V-bind}}=\mathcal{L}_{v,t}+\mathcal{L}_{v,a}$, where $\mathcal{L}_{\alpha,\beta}$ denotes the wBCE loss between modalities $\alpha$ and $\beta$ following~\citep{xu2024low}. Unlike these pairwise objectives, HoPA directly distills the joint omnimodal relationship among all modalities.
As shown in Table~\ref{tab:pair_vs_omni_main}, HoPA consistently achieves the best Avg. R@1 across all three datasets and query budgets, and remains competitive on Avg. R@5. Moreover, pair-wise objectives generally become less stable as the number of synthetic queries increases, especially on MSR-VTT and DiDeMo, where their performance often stagnates or degrades under larger query budgets. In contrast, HoPA remains stable and often improves further, suggesting that preserving the joint omnimodal structure is more effective and robust than decomposing it into separate pairwise alignments.

\begin{figure*}[t]
    \centering
    \includegraphics[width=0.95\linewidth]{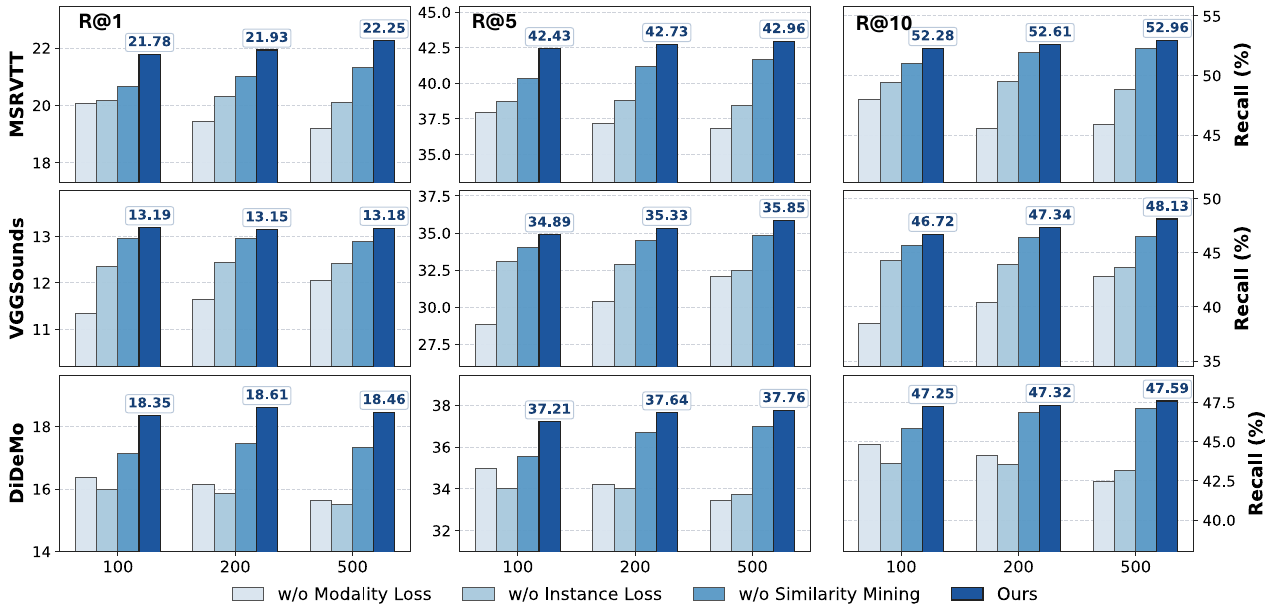}
    \caption{\textbf{Ablation study} of each learning objective across all datasets and different numbers of queries.}
    \label{fig:ablation_objective}
\end{figure*}

\textbf{Effect of each learning objective.} To study the contribution of each component, we conduct an ablation study on our learning objective. As shown in Figure~\ref{fig:ablation_objective}, removing either the modality-level loss $\gL_\mathrm{M}$ or instance-level loss $\gL_\text{wBCE}$ consistently degrades performance, indicating that both objectives are essential. Moreover, similarity mining further improves the results on top of these objectives, and the full model achieves the best overall performance.

\subsection{Further Analysis (RQ4)}

\begin{table}[t]
\centering
\caption{\textbf{Effect of the proxy rank on distillation performance.} We compare \textit{Rank-1} (Ours) against  \textit{Rank-2} across different metrics. In each case, the best result is shown in bold. As can be seen,  \textit{Rank-1} consistently outperforms \textit{Rank-2}.} 
\label{tab:rank_analysis}
\resizebox{0.75\linewidth}{!}{
\begin{tabular}{lc|cc|cc|cc}
\toprule
\multirow{2}{*}{\textbf{Dataset}} & \multirow{2}{*}{\textbf{Triplets}} & \multicolumn{2}{c|}{\textbf{R@1}} & \multicolumn{2}{c|}{\textbf{R@5}} & \multicolumn{2}{c}{\textbf{R@10}} \\
\cmidrule(lr){3-4} \cmidrule(lr){5-6} \cmidrule(lr){7-8}
 & & \textit{Rank-1} & \textit{Rank-2} & \textit{Rank-1} & \textit{Rank-2} & \textit{Rank-1} & \textit{Rank-2} \\
\midrule
\multirow{3}{*}{MSR-VTT} 
 & 100 & \textbf{21.78} & 20.80 & \textbf{42.43} & 39.99 & \textbf{52.28} & 49.77 \\
 & 200 & \textbf{21.93} & 21.09 & \textbf{42.73} & 39.81 & \textbf{52.61} & 49.37 \\
 & 500 & \textbf{22.25} & 20.48 & \textbf{42.96} & 39.78 & \textbf{52.96} & 49.57 \\
\midrule
\multirow{3}{*}{VGGSound-S} 
 & 100 & \textbf{13.19} & 12.18 & \textbf{34.89} & 32.85 & \textbf{46.72} & 44.13 \\
 & 200 & \textbf{13.15} & 12.25 & \textbf{35.33} & 32.58 & \textbf{47.34} & 43.72 \\
 & 500 & \textbf{13.18} & 12.44 & \textbf{35.85} & 32.54 & \textbf{48.13} & 43.34 \\
\midrule
\multirow{3}{*}{DiDeMo} 
 & 100 & \textbf{18.35} & 16.67 & \textbf{37.21} & 35.13 & \textbf{47.25} & 44.54 \\
 & 200 & \textbf{18.61} & 17.01 & \textbf{37.64} & 35.52 & \textbf{47.32} & 45.21 \\
 & 500 & \textbf{18.46} & 16.80 & \textbf{37.76} & 35.25 & \textbf{47.59} & 44.36 \\
\bottomrule
\end{tabular}
}
\end{table}

\textbf{Preserving more ranks.} A core premise of HoPA is that the perfect omnimodal alignment corresponds to a rank-1 Gram matrix (\textit{a.k.a.}, \textit{Rank-1}). To verify this empirically, we investigate whether preserving more information from SVD decomposition improves distillation performance. Specifically, we modify the proxy definition to include the top-2 singular vectors instead of just the leading one (\textit{a.k.a.}, \textit{Rank-2}). The results are reported in Table~\ref{tab:rank_analysis}.
We observe a consistent performance drop when increasing the proxy rank from 1 to 2 across all datasets. This phenomenon indicates that the leading singular direction ($\mathbf{v}_1$) captures the consensus across modalities, while the subsequent directions capture orthogonal components, which typically correspond to modality-specific details or noises.

\begin{wraptable}{r}{0.5\textwidth}
\vspace{-4mm}
\centering
\caption{\textbf{Computational efficiency analysis} by comparing the runtime with and without the SVD on MSR-VTT. The distillation time is measured per 100 iterations. }
\label{tab:efficiency_svd}
\resizebox{\linewidth}{!}{
\begin{tabular}{l|cc|c}
\toprule
\textbf{Stage} & \textbf{w/o SVD} & \textbf{with SVD} & \textbf{SVD Cost} \\
\midrule
Buffer Training (min/traj) & 10.85 & 11.25 & +3.5\% \\
Distillation (sec/100 iters) & 53.28 & 87.48 & +39.1\% \\
\bottomrule
\end{tabular}
}
\vspace{-3mm}
\end{wraptable}
\textbf{Efficiency analysis.}
A potential concern with our proposed method is the computational cost introduced by the SVD operation. To provide a quantitative assessment, we benchmark the training and distillation time of HoPA against a baseline where SVD is replaced by computationally negligible tensor manipulations (\textit{i.e.}, identity mapping and mean pooling) while maintaining the same gradient graph structure. The experiments are conducted on the MSR-VTT dataset with the same hyperparameters used in our main results.
As shown in Table~\ref{tab:efficiency_svd}, the inclusion of SVD introduces a negligible latency (+3.5\%) during the buffer training phase. In the distillation phase, the overhead increases to 39.1\%. This is attributed to the TESLA~\citep{cui2023scaling} algorithm we use in practice to reduce the memory consumption, which back-propagates through the SVD operation involving high-order derivative calculations.

\textbf{Case visualization.} 
We provide qualitative visualization for omnimodal dataset distillation on ImageBind using the VGGSound-S dataset. Figure~\ref{fig:visualization_main} shows both a representative comparison between an original instance and its distilled counterpart on the left, and additional distilled cases with richer visualization on the right. 
For clarity, the comparison view uses one representative video frame and one audio map, while each case visualization contains three video frames and three audio maps.
Consistent with prior observations in dataset distillation, the synthetic visual pattern shows a DeepDream-style appearance~\citep{zeiler2014visualizing}, with enhanced high-frequency details and stylized textures. 
For the text modality, we follow the same retrieval protocol as prior work~\citep{xu2024low,zhang2025beyond}, \textit{i.e.}, we display the closest sentence retrieved from the training set according to the distilled text embedding. 
The example indicates that our distilled omnimodal sample preserves the core event semantics and cross-modal correspondence across video, audio, and language in a compact synthetic form. 
More examples are presented in Appendix~\ref{appendix:visualization}.

\begin{figure}
    \centering
    \includegraphics[width=\linewidth]{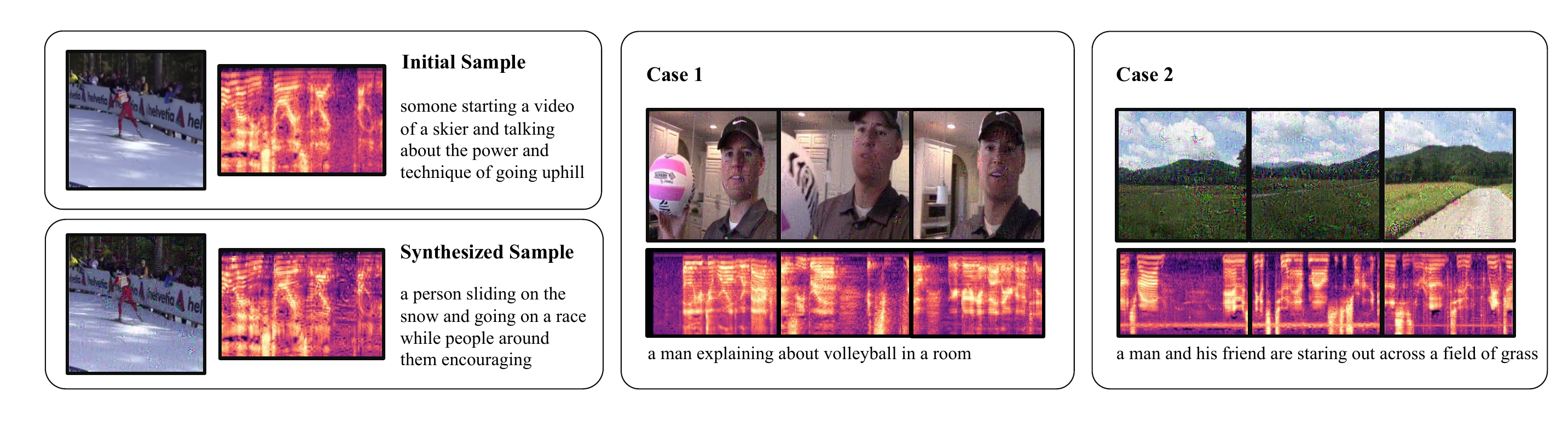}
    \vspace{-6mm}
    \caption{\textbf{The visualization of distilled omnimodal data on ImageBind with the VGGSound-S dataset.} Comparison between samples before and after distillation, shown with one representative video frame, one audio map, and the corresponding text (\textit{left}). Illustrative cases, each visualized with three video frames, three audio maps, and the corresponding text description (\textit{right}).}
    \label{fig:visualization_main}
\end{figure}

%% file: sec/5_conclusion.tex
\section{Conclusion}
\label{sec:conclusion}
In this work, we investigated dataset distillation in omnimodal settings beyond two modalities, which is a problem that has received little prior attention. We proposed a unified omnimodal dataset distillation method, termed HoPA, that leverages a compact proxy to model high-order alignments, enabling efficient and scalable distillation without explicit pairwise modality enumeration. 
Theoretically, we identify an overlooked determinant that bounds the endpoint discrepancy arising from the alignment objective, which motivates our design. The spectral analysis further demonstrates that HoPA induces a tighter bound compared to the pairwise dataset distillation methods, effectively minimizing this discrepancy.
Extensive experimental results demonstrate that our HoPA consistently achieves strong compression efficiency while preserving downstream performance across diverse datasets. These findings highlight the feasibility and promise of omnimodal dataset distillation at scale, and open new directions for data-efficient learning in complex omnimodal systems.

%% file: sec/appendix.tex
\section{Theoretical Analysis}
\label{appendix:theoretial_ana}

\subsection{Proof of Lemma~\ref{lem:local_to_segment_main}}
\label{app:proof_lemma1}

\begin{proof}
We track the parameter gap $\delta_r := \boldsymbol{\theta}_r^S - 
\boldsymbol{\theta}_r^T$ over rollout steps $r = 0, \ldots, n-1$.

\textbf{Step 1: One-step recursion.}
By the gradient descent updates,
\begin{align}
\delta_{r+1}
= \delta_r
- \eta\bigl(\mathbf{g}_S(\boldsymbol{\theta}_r^S)
- \mathbf{g}_T(\boldsymbol{\theta}_r^T)\bigr).
\end{align}
We decompose the gradient difference as
\begin{align}
\mathbf{g}_S(\boldsymbol{\theta}_r^S) - \mathbf{g}_T(\boldsymbol{\theta}_r^T)
=
\underbrace{
\bigl(\mathbf{g}_S(\boldsymbol{\theta}_r^S)
- \mathbf{g}_T(\boldsymbol{\theta}_r^S)\bigr)
}_{\text{data mismatch}}
+
\underbrace{
\bigl(\mathbf{g}_T(\boldsymbol{\theta}_r^S)
- \mathbf{g}_T(\boldsymbol{\theta}_r^T)\bigr)
}_{\text{Lipschitz term}}.
\end{align}

\textbf{Step 2: Apply the Lipschitz condition.}
By the $L$-Lipschitz assumption on $\mathbf{g}_T$,
\begin{align}
\|\mathbf{g}_T(\boldsymbol{\theta}_r^S)
- \mathbf{g}_T(\boldsymbol{\theta}_r^T)\|
\le L\|\boldsymbol{\theta}_r^S - \boldsymbol{\theta}_r^T\|
= L\|\delta_r\|.
\end{align}

\textbf{Step 3: One-step norm bound.}
Taking norms and applying the triangle inequality,
\begin{align}
\|\delta_{r+1}\|
\le \|\delta_r\|
+ \eta\|\mathbf{g}_S(\boldsymbol{\theta}_r^S)
- \mathbf{g}_T(\boldsymbol{\theta}_r^S)\|
+ \eta L\|\delta_r\|
= (1 + \eta L)\|\delta_r\|
+ \eta\Delta_r,
\end{align}
where
\begin{align}
\Delta_r
:= \|\mathbf{g}_S(\boldsymbol{\theta}_r^S)
- \mathbf{g}_T(\boldsymbol{\theta}_r^S)\|
= \left(\sum_{m \in \mathcal{M}}
\|\mathbf{g}_{S,m}(\boldsymbol{\theta}_r^S)
- \mathbf{g}_{T,m}(\boldsymbol{\theta}_r^S)\|_2^2
\right)^{\frac{1}{2}}.
\end{align}

\textbf{Step 4: Unroll the recursion.}
Since $\delta_0 = \boldsymbol{\theta}_0^S - \boldsymbol{\theta}_0^T = 
\mathbf{0}$ by shared initialization, unrolling the recursion gives
\begin{align}
\|\delta_n\|
\le \sum_{r=0}^{n-1}
\eta(1+\eta L)^{n-1-r}\Delta_r
= \eta\sum_{r=0}^{n-1}(1+\eta L)^{n-1-r}
\left(\sum_{m\in\mathcal{M}}
\|\mathbf{g}_{S,m}(\boldsymbol{\theta}_r^S)
-\mathbf{g}_{T,m}(\boldsymbol{\theta}_r^S)\|_2^2
\right)^{\frac{1}{2}},
\end{align}
which is the stated bound.
\end{proof}

\subsection{Spectral Forms of Pairwise InfoNCE and Pairwise Weighted BCE}
\label{app:spectralize_pairwise_losses}
We show that pairwise InfoNCE and pairwise weighted BCE are full-spectrum: for each instance $i$ and each active
mode $j$ ($\sigma_j^{(i)}>0$), the projection $\beta_j^{(i)}$ is nonzero for generic representations, \textit{i.e.}, it vanishes only under special alignment conditions that do not hold in general.
Recall
\begin{align}
s_{ab}^{(i,l)}=\langle \mathbf{z}_a^{(i)},\mathbf{z}_b^{(l)}\rangle
=\mathbf{e}_a^\top \mathbf{z}^{(i)}(\mathbf{z}^{(l)})^\top \mathbf{e}_b.
\end{align}

Using $\mathbf{z}^{(i)}\mathbf{v}_j^{(i)}=\sigma_j^{(i)}\mathbf{u}_j^{(i)}$, we have
\begin{equation}
\langle \mathbf{v}_j^{(i)},\mathbf{z}_a^{(i)}\rangle
=\sigma_j^{(i)}[\mathbf{u}_j^{(i)}]_a,
\qquad
\langle \mathbf{v}_j^{(i)},\mathbf{z}_b^{(i)}\rangle
=\sigma_j^{(i)}[\mathbf{u}_j^{(i)}]_b.
\label{eq:svd_identity}
\end{equation}

Across the batch, $\mathbf{z}^{(i)}$ appears in two roles: anchor in $s_{ab}^{(i,l)}$, candidate in $s_{ab}^{(m,i)}$. Hence
\begin{equation}
\frac{\partial s_{ab}^{(i,l)}}{\partial \mathbf{z}^{(i)}} = \mathbf{e}_a(\mathbf{z}_b^{(l)})^\top,
\qquad
\frac{\partial s_{ab}^{(m,i)}}{\partial \mathbf{z}^{(i)}} = \mathbf{e}_b(\mathbf{z}_a^{(m)})^\top,
\label{eq:grad_roles}
\end{equation}
both holding for all $m,l$ including $m=l=i$.

For either pairwise loss, define
$g_{ml}:=\partial \gL/\partial s_{ab}^{(m,l)}$, and assume $(i,i)$ is always included in the pair set so that $g_{ii}$ is well-defined.
By the chain rule and Eq.~(\ref{eq:grad_roles}),
\begin{equation}
\nabla_{\mathbf{z}^{(i)}}\gL
= \sum_{l} g_{il}\,\mathbf{e}_a(\mathbf{z}_b^{(l)})^\top
+\sum_m g_{mi}\,\mathbf{e}_b(\mathbf{z}_a^{(m)})^\top.
\label{eq:grad_generic}
\end{equation}
Projecting onto mode $j$ via Eq.~(\ref{eq:beta_proj}):
\begin{align}
\beta_j^{(i)}
&=\mathbf{u}_j^{(i)\top}\!\bigl(\nabla_{\mathbf{z}^{(i)}}\gL\bigr)\mathbf{v}_j^{(i)}
\nonumber\\
&=[\mathbf{u}_j^{(i)}]_a\sum_{l} g_{il}\,\langle \mathbf{v}_j^{(i)},\mathbf{z}_b^{(l)}\rangle
+[\mathbf{u}_j^{(i)}]_b\sum_m g_{mi}\,\langle \mathbf{v}_j^{(i)},\mathbf{z}_a^{(m)}\rangle
\nonumber\\
&=2g_{ii}\sigma_j^{(i)}[\mathbf{u}_j^{(i)}]_a[\mathbf{u}_j^{(i)}]_b
+[\mathbf{u}_j^{(i)}]_a\sum_{l\neq i} g_{il}\,\langle \mathbf{v}_j^{(i)},\mathbf{z}_b^{(l)}\rangle
+[\mathbf{u}_j^{(i)}]_b\sum_{m\neq i} g_{mi}\,\langle \mathbf{v}_j^{(i)},\mathbf{z}_a^{(m)}\rangle,
\label{eq:beta_generic}
\end{align}
where the last step separates out the $l=i$ and $m=i$ terms and applies Eq.~(\ref{eq:svd_identity}) to each: the $l=i$ anchor term contributes $g_{ii}\sigma_j^{(i)}[\mathbf{u}_j^{(i)}]_a[\mathbf{u}_j^{(i)}]_b$ and the $m=i$ candidate term contributes the same, summing to the factor of~$2$.

\textbf{Pairwise InfoNCE.}
\begin{align}
\gL_{\mathrm{InfoNCE}}
=-\frac{1}{N}\sum_{m=1}^N
\log\frac{\exp(s_{ab}^{(m,m)}/\tau)}
{\sum_{l=1}^N\exp(s_{ab}^{(m,l)}/\tau)},
\qquad
g_{ml}=\frac{p_{ml}-\mathbb{I}[m=l]}{N\tau},
\end{align}
where $p_{ml}=\exp(s_{ab}^{(m,l)}/\tau)\big/\sum_t \exp(s_{ab}^{(m,t)}/\tau)$.
Note that all pairs $(m,l)$ are present in the softmax, so $g_{ii}$ is always well-defined.
Substituting into Eq.~(\ref{eq:beta_generic}):
\begin{align}
\beta_j^{(i)}
=\frac{1}{N\tau}\Bigl[
&2(p_{ii}-1)\sigma_j^{(i)}[\mathbf{u}_j^{(i)}]_a[\mathbf{u}_j^{(i)}]_b
\nonumber\\
&+[\mathbf{u}_j^{(i)}]_a\sum_{l\neq i} p_{il}\,\langle\mathbf{v}_j^{(i)},\mathbf{z}_b^{(l)}\rangle
+[\mathbf{u}_j^{(i)}]_b\sum_{m\neq i} p_{mi}\,\langle\mathbf{v}_j^{(i)},\mathbf{z}_a^{(m)}\rangle
\Bigr].
\label{eq:beta_infonce_final}
\end{align}
For finite $\tau$ and $N\ge2$, $p_{ii}<1$, so the diagonal coefficient $2(p_{ii}-1)/(N\tau)$ is strictly negative.
The diagonal contribution is therefore nonzero whenever $[\mathbf{u}_j^{(i)}]_a[\mathbf{u}_j^{(i)}]_b\neq0$, which holds for generic representations.

\textbf{Pairwise wBCE.}
\begin{align}
\gL_{\mathrm{wBCE}}
=\frac{1}{|\mathcal P|}\sum_{(m,l)\in\mathcal P}
w_{ml}\,\ell_{\mathrm{BCE}}\!\left(y_{ml},\,\mathrm{sigmoid}(s_{ab}^{(m,l)}/\tau')\right),
\qquad
g_{ml}=\frac{w_{ml}r_{ml}}{|\mathcal{P}|\tau'},
\end{align}
with $r_{ml}:=\sigma(s_{ab}^{(m,l)}/\tau')-y_{ml}$, $w_{ml}>0$, and $\ell_{\mathrm{BCE}}(y,\hat y):=-y\log\hat y-(1-y)\log(1-\hat y)$ denoting
the binary cross-entropy.
Substituting into Eq.~(\ref{eq:beta_generic}):
\begin{align}
\beta_j^{(i)}
=\frac{1}{|\mathcal{P}|\tau'}\Bigl[
&2w_{ii}r_{ii}\sigma_j^{(i)}[\mathbf{u}_j^{(i)}]_a[\mathbf{u}_j^{(i)}]_b
\nonumber\\
&+[\mathbf{u}_j^{(i)}]_a\sum_{l\neq i} w_{il}r_{il}\,\langle\mathbf{v}_j^{(i)},\mathbf{z}_b^{(l)}\rangle
+[\mathbf{u}_j^{(i)}]_b\sum_{m\neq i} w_{mi}r_{mi}\,\langle\mathbf{v}_j^{(i)},\mathbf{z}_a^{(m)}\rangle
\Bigr].
\label{eq:beta_wbce}
\end{align}
By assuming $(i,i)\in\mathcal{P}$, $y_{ii}=1$, $\tau'>0$, and $\ell_2$-normalized features ($s_{ab}^{(i,i)}\in[-1,1]$), we have $r_{ii}=\sigma(s_{ab}^{(i,i)}/\tau')-1<0$.
Hence, the diagonal coefficient $2w_{ii}r_{ii}\sigma_j^{(i)}/(|\mathcal{P}|\tau')$ is strictly negative (since $w_{ii}>0$ and $\sigma_j^{(i)}>0$), and nonzero whenever $[\mathbf{u}_j^{(i)}]_a[\mathbf{u}_j^{(i)}]_b\neq0$.

Therefore, both pairwise InfoNCE and pairwise weighted BCE are full-spectrum in general. They activate all active spectral modes, accumulating gradient mismatch from every direction, in contrast to our proposed single-mode objective.

\subsection{Derivation of the Spectral Mismatch Model}
\label{app:spectral_mismatch_derivation}
We derive the per-step gradient mismatch bound and
Theorem~\ref{thm:eigen_selectivity} from first principles.
Throughout, $\mathbf{z} \in \mathbb{R}^{k \times d}$ is the stacked representation matrix for a single instance with thin SVD $\mathbf{z} = \sum_{j=1}^{K} \sigma_j \mathbf{u}_j \mathbf{v}_j^{\top}$,
and $\mathbf{J} := \partial\,\mathrm{vec}(\mathbf{z}) / \partial\boldsymbol{\theta} \in \mathbb{R}^{kd \times P}$ is the encoder Jacobian, where $P$ is the number of model parameters.

\begin{assumption}
\label{assum:spectral}
(i) \textit{(Differentiability)}
    The inner objective $\gL = f(\sigma_1, \ldots, \sigma_K)$ is differentiable with respect to each $\sigma_j$, at points of distinct singular values.
(ii) \textit{(Bounded Jacobian)}
    $\|\mathbf{J}(\boldsymbol{\theta})\|_{\mathrm{op}} \leq C$ uniformly over all $\boldsymbol{\theta}$ considered.
\end{assumption}
\begin{proof}
\textbf{Step 1: Representation-level gradient.}
By the standard matrix singular-value derivative~\citep{ionescu2015matrix}, at points of distinct singular values, $\partial \sigma_j / \partial \mathbf{z} = \mathbf{u}_j \mathbf{v}_j^{\top}$.
By the chain rule and Assumption~\ref{assum:spectral}(i),
\begin{equation}
\label{eq:rep_grad}
    \nabla_{\mathbf{z}} f
    = \sum_{j=1}^{K} \alpha_{j,r}\, \mathbf{u}_j \mathbf{v}_j^{\top},
    \qquad \alpha_{j,r} := \frac{\partial f}{\partial \sigma_j}\bigg|_r.
\end{equation}

\textbf{Step 2: Parameter-level gradient mismatch.}
By the chain rule through the encoder,
\begin{align}
    \mathbf{g}_{T}(\boldsymbol{\theta}_r^{S})
    = \mathbf{J}(\boldsymbol{\theta}_r^{S})^{\top}
      \,\mathrm{vec}\!\left(\nabla_{\mathbf{z}} f\big|_{T}\right),
    \qquad
    \mathbf{g}_{S}(\boldsymbol{\theta}_r^{S})
    = \mathbf{J}(\boldsymbol{\theta}_r^{S})^{\top}
      \,\mathrm{vec}\!\left(\nabla_{\mathbf{z}} f\big|_{S}\right).
\end{align}
Since both gradients are evaluated at the \textit{same} parameter point $\boldsymbol{\theta}_r^{S}$, they share an identical Jacobian $\mathbf{J}(\boldsymbol{\theta}_r^{S})$;
The sole difference is the data fed through the encoder, yielding representations $\mathbf{z}^T$ and $\mathbf{z}^S$ with corresponding singular vectors $\{\mathbf{u}_j^T, \mathbf{v}_j^T\}$ and $\{\mathbf{u}_j^S, \mathbf{v}_j^S\}$.
We evaluate $\alpha_{j,r}$ at the teacher representations and treat it as shared, which is accurate when the representation gap $\|\mathbf{z}^S - \mathbf{z}^T\|_F$ is small.
Applying Eq.~(\ref{eq:rep_grad}), the gradient difference becomes
\begin{equation}
\label{eq:grad_diff}
    \mathbf{g}_{S}(\boldsymbol{\theta}_r^{S}) -
    \mathbf{g}_{T}(\boldsymbol{\theta}_r^{S})
    = \mathbf{J}(\boldsymbol{\theta}_r^{S})^{\top}
      \,\mathrm{vec}\!\left(
          \sum_{j=1}^{K} \alpha_{j,r}
          \bigl(\mathbf{u}_j^{S}\mathbf{v}_j^{S\top}
          - \mathbf{u}_j^{T}\mathbf{v}_j^{T\top}\bigr)
      \right).
\end{equation}

\textbf{Step 3: Norm bound.}
Taking the Euclidean norm of Eq.~(\ref{eq:grad_diff}), applying submultiplicativity of the operator norm, the identity $\|\mathrm{vec}(\mathbf{M})\|_2 = \|\mathbf{M}\|_F$, and the triangle inequality in sequence,
\begin{align}
\label{eq:spectral_mismatch_app}
    \Delta_r
    \,&:=\,
    \left\|
        \mathbf{g}_{S}(\boldsymbol{\theta}_r^{S}) -
        \mathbf{g}_{T}(\boldsymbol{\theta}_r^{S})
    \right\|_2\\
    &\leq \|\mathbf{J}\|_{\mathrm{op}}
    \left\|\sum_{j=1}^{K} \alpha_{j,r}
    \bigl(\mathbf{u}_j^{S}\mathbf{v}_j^{S\top}
    - \mathbf{u}_j^{T}\mathbf{v}_j^{T\top}\bigr)\right\|_F \notag \leq C \sum_{j=1}^{K} |\alpha_{j,r}|\,\varepsilon_{j,r},
\end{align}
where $\varepsilon_{j,r} := \|\mathbf{u}_j^{S}\mathbf{v}_j^{S\top} - \mathbf{u}_j^{T}\mathbf{v}_j^{T\top}\|_F$ is the mode-$j$ approximation error, and $C$ is the constant from Assumption~\ref{assum:spectral}(ii).
\end{proof}

\subsection{Proof of Theorem~\ref{thm:eigen_selectivity}}
\label{app:proof_prop_eigen}
\begin{proof}
\textbf{Step 1: Lemma~\ref{lem:local_to_segment_main} backbone.}
Applying Lemma~\ref{lem:local_to_segment_main},
\begin{equation}
\label{eq:app_backbone}
\|\boldsymbol{\theta}_n^S - \boldsymbol{\theta}_n^T\|_2
\le \eta\sum_{r=0}^{n-1}(1+\eta L)^{n-1-r}\Delta_r.
\end{equation}

\textbf{Step 2: Substitute the spectral mismatch model.}
Inserting Eq.~(\ref{eq:spectral_mismatch_app}) into Eq.~(\ref{eq:app_backbone}),
\begin{equation}
\label{eq:app_general}
\|\boldsymbol{\theta}_n^S - \boldsymbol{\theta}_n^T\|_2
\le \eta C\sum_{r=0}^{n-1}(1+\eta L)^{n-1-r}
\sum_{j=1}^K |\alpha_{j,r}|\,\varepsilon_{j,r},
\end{equation}
establishing Eq.~(\ref{eq:bound_general_main}).

\textbf{Step 3: Comparison $U_A \le U_B$.}
Since $\alpha^A_{1,r}=\alpha^B_{1,r}$ for all $r$ and
$\alpha^A_{j,r}=0$ for all $j\ge 2$, at each step $r$:
\begin{align}
|\alpha^A_{1,r}|\,\varepsilon_{1,r}
= |\alpha^B_{1,r}|\,\varepsilon_{1,r}
\le \sum_{j=1}^K |\alpha^B_{j,r}|\,\varepsilon_{j,r}.
\end{align}
Multiplying by the positive weight $\eta C(1+\eta L)^{n-1-r}$ and summing over $r=0,\ldots,n-1$ gives $U_A\le U_B$.
The inequality is strict whenever $|\alpha^B_{j,r}|\,\varepsilon_{j,r}>0$ for some $j\ge 2$ at some step $r$.
\end{proof}

\section{Supplementary Implementation Details}
\subsection{Datasets}\label{sec:dataset_statistics}
In the main paper, we evaluate our method on MSR-VTT~\citep{xu2016msr}, VGGSound-S~\citep{chen2020vggsound}, and DiDeMo~\citep{anne2017didemo}. The details of the datasets are summarized below. 
\begin{itemize}
    \item MSR-VTT is a benchmark from Microsoft Research for video captioning. It contains 10,000 web video clips and about 200,000 clip‑sentence pairs spanning 20 categories, with roughly 20 captions per clip. We use the split MSR-VTT 1kA following Imagebind~\citep{girdhar2023imagebind}.
    \item VGGSound-S is a curated subset of the original VGGSound dataset. VGGSound is an audio–visual benchmark created by the University of Oxford’s Visual Geometry Group. It consists of more than 200,000 ten‑second YouTube clips, each depicting a sound source that is visible in the video and labeled with one of 310+ sound‑event classes. The full collection spans over 550 hours of audio‑visual content. To construct VGGSound‑S, we randomly sampled 10,000 examples from the official training set and 2,000 examples from the validation set, yielding a smaller dataset that preserves the diversity of the original while enabling more efficient multi‑modal retrieval experiments.
    \item DiDeMo is a collection of more than 10,000 unedited personal videos scraped from Flickr, each roughly 25–30s long with an accompanying audio track.  Each video comes with multiple natural‑language descriptions (about 3–5 per clip) written by annotators. These descriptions were validated so that they uniquely match the described clip. In total, the dataset offers more than 40,000 (description, video) pairs and spans a wide range of real‑world content such as pets, concerts, and sports. We use the official split, which contains 8,395 training videos, 1,065 validation videos, and 1,004 test videos.
\end{itemize}

\subsection{Hyperparameter Settings}\label{sec:hyperparameters}
The hyperparameter settings are summarized in Table~\ref{tab:buffer_settings} and Table~\ref{tab:distillation_settings}. 

\begin{table}[h]
\centering
\caption{\textbf{Hyperparameter settings for buffer.}}
\label{tab:buffer_settings}
\begin{tabular}{l|ccc}
\toprule
Dataset & MSR-VTT & VGGSound-S & DiDeMo \\ 
\midrule
epochs & 10 & 10 & 10 \\
num\_experts & 20 & 20 & 20 \\
batch\_size & 128 & 128 & 128 \\
lr\_teacher & 0.01 & 0.01 & 0.01 \\
$\tau$ & 0.1  & 0.1  & 0.1 \\
$\tau'$ & 0.2  & 0.2  & 0.2 \\
\bottomrule
\end{tabular}
\end{table}

\begin{table}[t] 
\centering
\caption{\textbf{Hyperparameter settings for dataset distillation.}}
\label{tab:distillation_settings}
\begin{tabular}{l|ccc|ccc|ccc}
\toprule
Dataset & \multicolumn{3}{c|}{MSR-VTT} & \multicolumn{3}{c|}{VGGSound-S} & \multicolumn{3}{c}{DiDeMo} \\\midrule
\#instances & 100 & 200 & 500 & 100 & 200 & 500 & 100 & 200 & 500 \\
\midrule
lr\_data & 100 & 100 & 100 & 100 & 100 & 100 & 100 & 100 & 100 \\
lr\_lr & 1e-4 & 1e-4 & 1e-4 & 1e-4 & 1e-4 & 1e-4 & 1e-4 & 1e-4 & 1e-4 \\
lr\_sim & 10.0 & 10.0 & 100.0 & 10.0 & 10.0 & 100.0 & 10.0 & 10.0 & 100.0 \\
lr\_teacher & 0.01 & 0.01 & 0.01 & 0.01 & 0.01 & 0.01 & 0.01 & 0.01 & 0.01 \\
sim\_rank & 10 & 20 & 20 & 10 & 20 & 40 & 10 & 20 & 20 \\
sim\_alpha & 1.0 & 0.5 & 0.01 & 1.0 & 1.0 & 0.1 & 1.0 & 0.5 & 0.05 \\
mini\_batch\_size & 25 & 25 & 25 & 25 & 25 & 25 & 25 & 25 & 25 \\
\midrule
$\tau$ & 0.1 & 0.1 & 0.1 & 0.1 & 0.1 & 0.1 & 0.1 & 0.1 & 0.1 \\
$\tau'$ & 0.2 & 0.2 & 0.2 & 0.2 & 0.2 & 0.2 & 0.2 & 0.2 & 0.2 \\
\midrule
iterations & 2,000 & 2,000 & 3,000 & 2,000 & 2,000 & 3,000 & 2,000 & 2,000 & 3,000 \\
syn\_steps & 16 & 16 & 16 & 16 & 16 & 16 & 16 & 16 & 16 \\
expert\_epochs & 2 & 2 & 2 & 2 & 2 & 2 & 2 & 2 & 2 \\
max\_start\_epoch & 5 & 5 & 7 & 5 & 5 & 7 & 5 & 5 & 7 \\
\bottomrule
\end{tabular}
\end{table}

\subsection{Algorithm Flow}\label{sec:algorithm_flow}
For a better understanding of the technical details of our method, we provide an algorithm flow in Algorithm~\ref{alg:omni_dd}.

\begin{algorithm}[t]
\caption{The technical flow of the HoPA framework}
\label{alg:omni_dd}
\begin{algorithmic}[1]
\REQUIRE
    \textbf{Inputs:} \\
    \quad The real dataset $\gD = \{(\vx_i^{m_1}, \vx_i^{m_2}, \dots, \vx_i^{m_k})\}_{i=1}^N$ with $k = |\gM|$ modalities per instance. \\
    \quad Expert trajectory buffer $\mathcal{B}$. \\
    \quad Modality-specific encoders $\{E_m\}_{m\in\gM}$.\\
    \quad Student rollout steps $t$, segment length $T$, learning rates $\eta, \alpha$, the number of total iterations $Iter$.
\ENSURE Synthetic set $\gD_{\text{syn}}$ and learnable synthetic similarity matrix $\mS_e$.

\vspace{3pt}
\STATE Initialize synthetic data $\gD_{\text{syn}}$ and learnable parameters $\boldsymbol{\phi}_S$ for matrix $\mS_e$; 

\FOR{$it = 1$ to $Iter$}
    \STATE Sample a trajectory segment $(\boldsymbol{\theta}_0, \boldsymbol{\theta}_T) \sim \mathcal{B}$;
    \STATE Initialize student parameters $\boldsymbol{\theta}^{\mathrm{S}}_0 \leftarrow \boldsymbol{\theta}_0$; 

    \STATE \textcolor{darksilver}{// Simulate Training on Synthetic Data}
    \FOR{$r = 1$ to $t$}
        \STATE For each synthetic instance $i$, compute multimodal representations $\vz^{(i)} = [E_m(\vx_{\text{syn},i}^{m}; \boldsymbol{\theta}_{r-1,m}^{\mathrm{S}})]_{m\in\gM}^{\top}$; 
        
        \STATE For each instance $i$, compute SVD: $\vz^{(i)} = \sum_{l=1}^{k}\sigma_{l}^{(i)} \vu_{l}^{(i)} \vv_{l}^{(i)\top}$, and extract $\sigma_1^{(i)}$, $\vv_1^{(i)}$;
        
        \STATE  \textcolor{darksilver}{// Compute Combined Objective (Eq.~(\ref{eq:inner_target}))}
        \STATE $\gL_\mathrm{M} \leftarrow -\frac{1}{|\gD_{\text{syn}}|}\sum_i \log \frac{\exp(\sigma_1^{(i)}/\tau)}{\sum_{l=1}^{k}\exp(\sigma_{l}^{(i)}/\tau)}$; 
        \STATE Construct proxy similarity matrix with $\tilde{s}_{ij} = \vv_1^{(i)\top}\vv_1^{(j)}$; 
        \STATE $\gL_{\mathrm{wBCE}} \leftarrow \gL_{\mathrm{wBCE}}(\mS_e)$; 
        \STATE $\gL_{\text{inner}} \leftarrow \gL_\mathrm{M} + \gL_{\mathrm{wBCE}}$;
        
        \STATE $\boldsymbol{\theta}_{r}^{\mathrm{S}} \leftarrow \boldsymbol{\theta}_{r-1}^{\mathrm{S}} - \eta \nabla_{\boldsymbol{\theta}_{r-1}^{\mathrm{S}}} \gL_{\text{inner}}$;
    \ENDFOR

    \STATE  \textcolor{darksilver}{// Trajectory Matching (Eq.(~\ref{eq:omni_traj}))}
    \STATE $\gL_{\mathrm{match}} \leftarrow \sum_{m\in\gM}\frac{\|\boldsymbol{\theta}_{t,m}^{\mathrm{S}} - \boldsymbol{\theta}_{T,m}\|_2^2}{\|\boldsymbol{\theta}_{0,m} - \boldsymbol{\theta}_{T,m}\|_2^2}$; 
    \STATE Update $\gD_{\text{syn}}$ and $\boldsymbol{\phi}_S$ using gradient descent with learning rate $\alpha$ on $\gL_{\mathrm{match}}$;
\ENDFOR

\STATE \textbf{Return:} $\gD_{\text{syn}}$ and $\mS_e$.
\end{algorithmic}
\end{algorithm}

\section{Additional Results}
\subsection{Effect of Each Learning Objective}
\label{appendix:learning_objective}
We report the full results of the ablation study for each component of our learning objective in Table~\ref{tab:appendix_ablation_full}. 

\begin{table*}[h]
\centering
\caption{\textbf{Complete ablation study results.} We report Recall@1, Recall@5, and Recall@10 (\%) on MSR-VTT, VGGSound-S, and DiDeMo across varying synthetic sizes. The best results in each setting are marked in \textbf{bold}.}
\label{tab:appendix_ablation_full}

\begin{tabular}{cclccc}
\toprule
\textbf{Dataset} & \textbf{Triplets} & \textbf{Method Variant} & \textbf{R@1} & \textbf{R@5} & \textbf{R@10} \\
\midrule
\multirow{12}{*}{\textbf{MSR-VTT}} 
& \multirow{4}{*}{100} 
  & Ours (Full Model) & \textbf{21.78} & \textbf{42.43} & \textbf{52.28} \\
& & \textit{w/o} Similarity Mining & 20.67 & 40.30 & 51.02 \\
& & \textit{w/o} Instance Loss ($\gL_{\text{wBCE}}$) & 20.16 & 38.75 & 49.42 \\
& & \textit{w/o} Modality Loss ($\gL_\mathrm{M}$) & 20.08 & 37.96 & 48.03 \\
\cmidrule{2-6}
& \multirow{4}{*}{200} 
  & Ours (Full Model) & \textbf{21.93} & \textbf{42.73} & \textbf{52.61} \\
& & \textit{w/o} Similarity Mining & 21.02 & 41.18 & 51.92 \\
& & \textit{w/o} Instance Loss ($\gL_{\text{wBCE}}$) & 20.31 & 38.80 & 49.48 \\
& & \textit{w/o} Modality Loss ($\gL_\mathrm{M}$) & 19.44 & 37.15 & 45.54 \\
\cmidrule{2-6}
& \multirow{4}{*}{500} 
  & Ours (Full Model) & \textbf{22.25} & \textbf{42.96} & \textbf{52.96} \\
& & \textit{w/o} Similarity Mining & 21.32 & 41.65 & 52.29 \\
& & \textit{w/o} Instance Loss ($\gL_{\text{wBCE}}$) & 20.11 & 38.42 & 48.79 \\
& & \textit{w/o} Modality Loss ($\gL_\mathrm{M}$) & 19.20 & 36.83 & 45.92 \\
\midrule
\midrule
\multirow{12}{*}{\textbf{VGGSound-S}} 
& \multirow{4}{*}{100} 
  & Ours (Full Model) & \textbf{13.19} & \textbf{34.89} & \textbf{46.72} \\
& & \textit{w/o} Similarity Mining & 12.96 & 34.03 & 45.72 \\
& & \textit{w/o} Instance Loss ($\gL_{\text{wBCE}}$) & 12.35 & 33.06 & 44.28 \\
& & \textit{w/o} Modality Loss ($\gL_\mathrm{M}$) & 11.35 & 28.81 & 38.51 \\
\cmidrule{2-6}
& \multirow{4}{*}{200} 
  & Ours (Full Model) & \textbf{13.15} & \textbf{35.33} & \textbf{47.34} \\
& & \textit{w/o} Similarity Mining & 12.96 & 34.49 & 46.46 \\
& & \textit{w/o} Instance Loss ($\gL_{\text{wBCE}}$) & 12.44 & 32.88 & 43.95 \\
& & \textit{w/o} Modality Loss ($\gL_\mathrm{M}$) & 11.65 & 30.41 & 40.37 \\
\cmidrule{2-6}
& \multirow{4}{*}{500} 
  & Ours (Full Model) & \textbf{13.18} & \textbf{35.85} & \textbf{48.13} \\
& & \textit{w/o} Similarity Mining & 12.89 & 34.84 & 46.50 \\
& & \textit{w/o} Instance Loss ($\gL_{\text{wBCE}}$) & 12.41 & 32.48 & 43.62 \\
& & \textit{w/o} Modality Loss ($\gL_\mathrm{M}$) & 12.06 & 32.04 & 42.81 \\
\midrule
\midrule
\multirow{12}{*}{\textbf{DiDeMo}} 
& \multirow{4}{*}{100} 
  & Ours (Full Model) & \textbf{18.35} & \textbf{37.21} & \textbf{47.25} \\
& & \textit{w/o} Similarity Mining & 17.13 & 35.57 & 45.82 \\
& & \textit{w/o} Instance Loss ($\gL_{\text{wBCE}}$) & 15.98 & 34.02 & 43.64 \\
& & \textit{w/o} Modality Loss ($\gL_\mathrm{M}$) & 16.36 & 34.97 & 44.81 \\
\cmidrule{2-6}
& \multirow{4}{*}{200} 
  & Ours (Full Model) & \textbf{18.61} & \textbf{37.64} & \textbf{47.32} \\
& & \textit{w/o} Similarity Mining & 17.45 & 36.72 & 46.83 \\
& & \textit{w/o} Instance Loss ($\gL_{\text{wBCE}}$) & 15.86 & 34.00 & 43.55 \\
& & \textit{w/o} Modality Loss ($\gL_\mathrm{M}$) & 16.15 & 34.20 & 44.12 \\
\cmidrule{2-6}
& \multirow{4}{*}{500} 
  & Ours (Full Model) & \textbf{18.46} & \textbf{37.76} & \textbf{47.59} \\
& & \textit{w/o} Similarity Mining & 17.33 & 36.99 & 47.13 \\
& & \textit{w/o} Instance Loss ($\gL_{\text{wBCE}}$) & 15.50 & 33.74 & 43.14 \\
& & \textit{w/o} Modality Loss ($\gL_\mathrm{M}$) & 15.64 & 33.42 & 42.47 \\
\bottomrule
\end{tabular}
\end{table*}

\section{Reproducibility}
We have provided implementation details, involving illustrative algorithm descriptions in Sec.~\ref{sec:method}  and pseudo-code in Algorithm~\ref{alg:omni_dd}. The source code will be publicly released for reproducibility.

\section{Visualization}
\label{appendix:visualization}
We provide visualizations of $N=100$ distilled video-audio-text instances from the VGGSound-S dataset in Figure~\ref{fig:visualization} before and after distillation. Following~\citep{wu2023vision}, the texts we present are retrieved by the closest matching captions in the train set to the distilled text embeddings.

\begin{figure}
    \centering
    \includegraphics[width=0.8\linewidth]{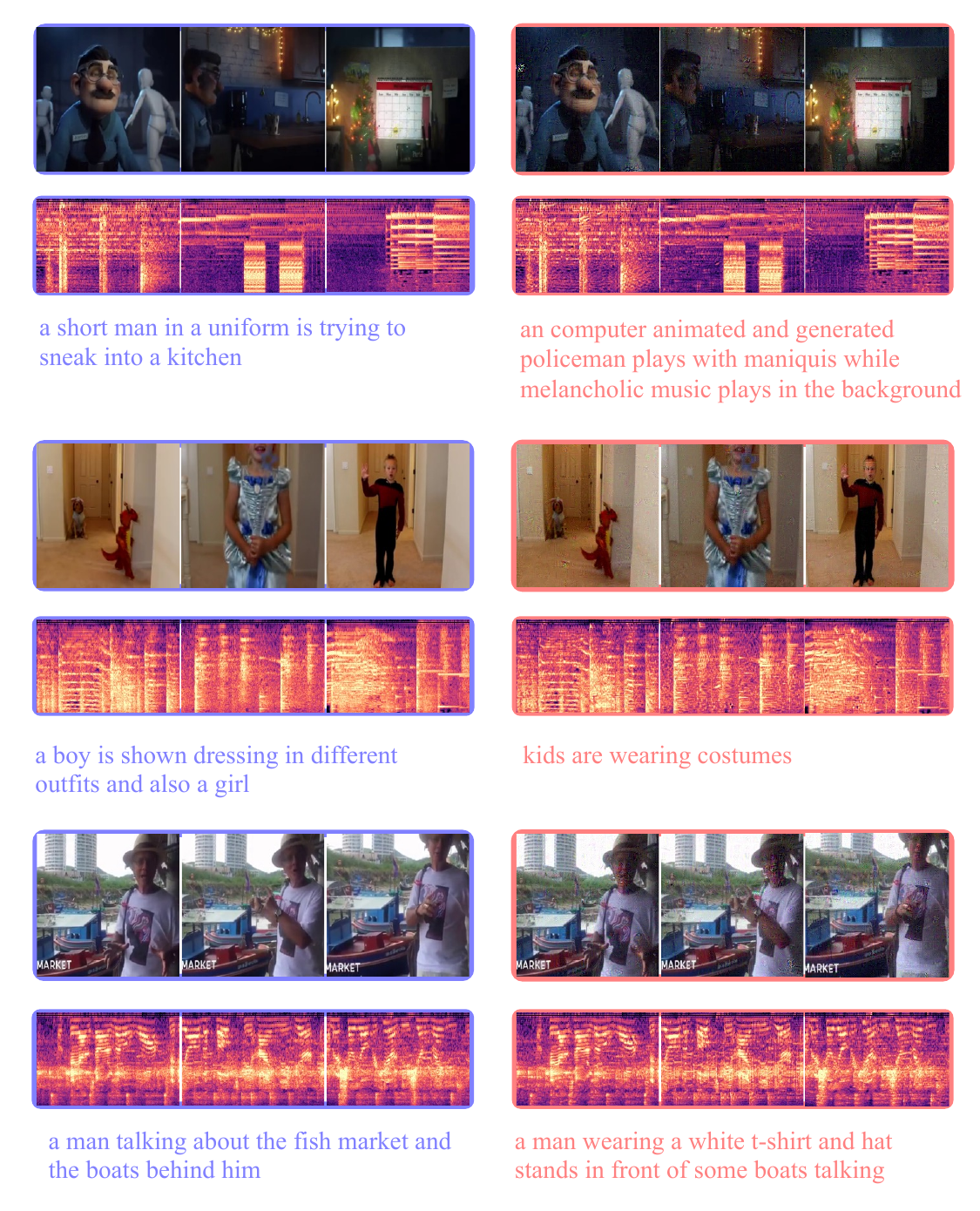}
    \caption{Examples of distilled instances from the VGGSound-S dataset with $N=100$ triplets.}
    \label{fig:visualization}
\end{figure}

\section{Limitations}
Although the proposed method is more efficient than existing multimodal baselines by avoiding exhaustive pairwise modality interactions, it still incurs non-negligible computational overhead due to the use of SVD for constructing the omnimodal proxy. In particular, SVD can become a bottleneck when scaling to very large datasets or a high number of modalities. Further improving the efficiency of the low-rank approximation~\cite{allen2016lazysvd}, for example, through incremental, randomized, or approximate decomposition techniques, is an important direction for future work.

\section{Use of LLMs in Writing}
We used an LLM solely to polish the writing and correct grammatical issues during the preparation of this submission. The LLM was not involved in any idea generation, experiment design, or analysis. All scientific contributions are entirely made by the authors.